\documentclass[times, review, 10pt]{elsarticle}
\usepackage{amssymb}
\usepackage{graphicx}
\usepackage{amsmath}
\usepackage{longtable}
\usepackage{a4wide}
\usepackage{mathrsfs}
\usepackage{subcaption}
\usepackage{mathtools}
\usepackage{pifont}
\usepackage{algorithmic}
\usepackage{algorithm}
\usepackage{xcolor}
\usepackage{color}
\usepackage{lineno}
\usepackage{makecell} 
\usepackage{pdflscape}
\usepackage{adjustbox}
\usepackage[utf8]{inputenc}
\usepackage{tabularx}
\usepackage{setspace}
\usepackage{blindtext}
\usepackage{multirow}
\usepackage{notoccite} 
\usepackage{lscape} 
\usepackage{caption} 
\usepackage{mwe}
\usepackage{tikz}
\usepackage{siunitx}
\usepackage{mathrsfs}
\usetikzlibrary{shapes,arrows}
\usepackage{xcolor}
\usepackage{booktabs}
\usepackage{hyperref}
\usepackage{amsthm}
\newtheorem{theorem}{Theorem}[section]

\theoremstyle{definition}
\newtheorem{definition}{Definition}[section]
\newcommand{\RNum}[1]{\lowercase\expandafter{\romannumeral #1\relax}}
\newcommand{\RNumU}[1]{\uppercase\expandafter{\romannumeral #1\relax}}
\usepackage{natbib}
\journal{Elsevier}

\begin{document}
\begin{frontmatter}
\title{Randomized based restricted kernel machine for hyperspectral image classification}
\author[inst1]{A. Quadir}
\ead{mscphd2207141002@iiti.ac.in}
\author[inst1]{M. Tanveer\corref{Correspondingauthor}}
\ead{mtanveer@iiti.ac.in}

\affiliation[inst1]{organization={Department of Mathematics, Indian Institute of Technology Indore},
            addressline={Simrol}, 
            city={Indore},
            postcode={453552}, 
            state={Madhya Pradesh},
            country={India}}
            \cortext[Correspondingauthor]{Corresponding author}
\begin{abstract}
In recent years, the random vector functional link (RVFL) network has gained significant popularity in hyperspectral image (HSI) classification due to its simplicity, speed, and strong generalization performance. However, despite these advantages, RVFL models face several limitations, particularly in handling non-linear relationships and complex data structures. The random initialization of input-to-hidden weights can lead to instability, and the model struggles with determining the optimal number of hidden nodes, affecting its performance on more challenging datasets. To address these issues, we propose a novel randomized based restricted kernel machine ($R^2KM$) model that combines the strehyperngths of RVFL and restricted kernel machines (RKM). $R^2KM$ introduces a layered structure that represents kernel methods using both visible and hidden variables, analogous to the energy function in restricted Boltzmann machines (RBM). This structure enables $R^2KM$ to capture complex data interactions and non-linear relationships more effectively, improving both interpretability and model robustness. A key contribution of $R^2KM$ is the introduction of a novel conjugate feature duality based on the Fenchel-Young inequality, which expresses the problem in terms of conjugate dual variables and provides an upper bound on the objective function. This duality enhances the model's flexibility and scalability, offering a more efficient and flexible solution for complex data analysis tasks. Extensive experiments on hyperspectral image datasets and real-world data from the UCI and KEEL repositories show that $R^2KM$ outperforms baseline models, demonstrating its effectiveness in classification and regression tasks.
\end{abstract}

\end{frontmatter}
\section{Introduction}
Hyperspectral images (HSIs) have garnered increased attention recently due to their rich spectral bands and wealth of spatial information \cite{cao2024diffusion}. HSIs contain a significantly larger volume of information compared to other remote sensing images, represented in the form of cubic data \cite{he2017recent}. Given the high value of processing and analyzing information from HSIs, researchers have developed a range of technical approaches \cite{he2021tslrln, luo2021dimensionality}, including hyperspectral classification. HSI has been extensively applied across various fields, such as medical imaging technology \cite{lv2021discriminant}, and mineral resource exploration \cite{wang2012spectral}, all of which depend on HSI classification. 

Over the past few decades, numerous traditional machine learning methods have been developed by researchers for HSI classification, including the K-nearest neighbor algorithm \cite{su2020random}, support vector machines (SVM) \cite{sheykhmousa2020support}, and random forests (RF) \cite{cong2021rrnet}. However, these conventional models fail to fully leverage spectral-spatial features to establish connections between pixels in the spatial dimension \cite{ahmad2021hyperspectral}. The development and successful application of various feature extraction algorithms have significantly enhanced the accuracy of HSI classification \cite{hao2020multiscale}. Some of the most notable handcrafted feature extraction techniques include mathematical morphological features such as extended morphological profiles (EMP) \cite{benediktsson2005classification}, morphological profiles (MP) \cite{fauvel2008spectral}, and extended multiattribute profiles (EMAP) \cite{dalla2010classification}. Moreover, feature extraction methods based on subspace learning \cite{he2021tslrln}, such as sparse and low-rank representations \cite{sun2020weighted}, have greatly enhanced the effectiveness of classification techniques. Furthermore, multi-view learning \cite{quadir2024multiview, quadir2024enhancing, houthuys2018multi}, and discriminant analysis \cite{ye2019nonpeaked} techniques can significantly improve feature representation. For a detailed overview of machine learning models for HSI classification, we refer interested readers to \cite{patel2023comprehensive}.

Artificial neural networks (ANNs) are computational models designed to mimic the architecture and operation of the neural networks found in the human brain. In ANNs, neurons are organized into layers that collaboratively process, analyze and communicate information, enabling the network to arrive at decisions. ANNs have achieved success in various fields, such as diagnosing Alzheimer’s disease \cite{tanveer2024ensemble, tanveer2024fuzzy, goel2025alzheimer}, stock market prediction \cite{lu2024trnn}, predicting DNA binding proteins \cite{quadir2024multiviewDAN}, solving differential equations \cite{berman2024randomized}, and more. Despite their benefits, ANN models encounter several challenges, including sensitivity to learning rates, slow convergence, and issues with local minima \cite{sajid2024intuitionistic}. To tackle these challenges, randomized neural networks (RNNs) \cite{suganthan2018non}, including random vector functional link (RVFL) network \cite{pao1994learning}, have been introduced. The RVFL model is characterized by a single hidden layer, where the weights connecting the input to the hidden layer are randomly generated and remain fixed during training. It also incorporates direct connections from the input layer to the output layer, which function as a built-in regularization mechanism, thereby enhancing the generalization capability of the RVFL \cite{zhang2016comprehensive}. The training process concentrates exclusively on adjusting the weights of the output layer, a task that can be efficiently performed using closed-form or iterative methods. Multiple improved variants of the traditional RVFL model have been developed to boost its generalization capabilities, thereby enhancing its robustness and effectiveness for real-world applications \cite{malik2023random}.

The RVFL model transforms the original features into random representations, which can introduce instability during the learning process. To tackle this challenge, modifications were made to the RVFL framework by integrating a sparse autoencoder with \( \ell_1 \)-norm regularization, leading to the development of the SP-RVFL model \cite{zhang2019unsupervised}. The SP-RVFL model mitigates the instability associated with randomization and demonstrates improved learning of network parameters when compared to the conventional RVFL approach. \citet{zhang2020new} presented two models: KRVFL+, which is a kernel-enhanced variant of RVFL+, and RVFL+, which combines the RVFL framework with learning utilizing privileged information (LUPI). During the training phase, RVFL+ leverages both the training data and additional privileged data. Along with incorporating privileged information, KRVFL+ effectively handles nonlinear interactions between higher-dimensional input and output vectors. RVFL+ networks face a similar issue as RVFL networks: determining the optimal number of hidden nodes remains a challenge. The effectiveness of the network's learning process is greatly influenced by the number of hidden nodes \cite{li2017bayesian}. A constructive algorithm known as incremental RVFL+ (IRVFL+) was developed by \cite{dai2022incremental}. By progressively adding hidden nodes, the IRVFL+ network enhances its approximation of the output. In \cite{chakravorti2020non}, the kernel-based exponentially expanded RVFL (KERVFL) is introduced, incorporating a kernel function into the RVFL model to eliminate the need for determining the optimal number of hidden nodes. \citet{suykens2017deep} introduced the restricted kernel machine (RKM) for classification and regression, with the goal of combining kernel methods with neural network techniques. This advancement broadens the application of kernel techniques, allowing them to address more complex real-world problems effectively. RKM employs Legendre-Fenchel duality \cite{rockafellar1974conjugate} to offer a representation of LSSVM \cite{suykens1999least} that is analogous to the energy function of a restricted Boltzmann machine (RBM) \cite{hinton2006fast}. Although the RKM has found successful applications in areas such as generative model \cite{pandey2022disentangled} and classification \cite{houthuys2021tensor, tao2024tensor, quadir2025trkm, quadir2025one}, its use in disentangled representations \cite{tonin2021unsupervised} and data exploration has not been thoroughly investigated in prior studies. RKM uses the kernel trick to map data into a high-dimensional feature space, enabling it to create a non-linear separating hyperplane that can effectively manage non-linear relationships. 

In this paper, we propose randomized based restricted kernel machine ($R^2KM$), a novel model designed to enhance generalization performance by combining the strengths of RVFL network and RKM. Building on the strengths of RVFL networks over conventional ANNs and machine learning models, our proposed $R^2KM$ integrates the computational efficiency of RVFL with the powerful feature transformation capabilities of RKM. This innovative fusion addresses the limitations of existing models by providing a comprehensive solution that improves both model performance and applicability across diverse and complex data scenarios. The $R^2KM$ effectively leverages the randomized feature transformation of RVFL while utilizing the robust kernel mechanisms of RKM, providing a more efficient and scalable solution for tackling complex data structures. The key highlights of this paper can be encapsulated as follows:
\begin{enumerate}
    \item We introduce randomized based restricted kernel machine ($R^2KM$), a novel model aimed at bridging the gap between traditional kernel methods and neural network-based approaches. By integrating the computational efficiency of RVFL networks with the robust feature mapping capabilities of RKM.
    \item The $R^2KM$ model is capable of representing kernel methods by employing both visible and hidden variables, akin to the functionality of the energy function found in restricted boltzmann machines (RBM). By incorporating this layered representation, $R^2KM$ not only enhances the interpretability of kernel-based models but also strengthens their capacity to handle challenging data patterns, offering a more robust and flexible approach for a wide range of classification and regression tasks.
    \item We employ a conjugate feature duality based on the Fenchel-Young inequality. This framework allows us to reformulate the $R^2KM$ problem with conjugate dual variables for all samples in the latent space. By leveraging this duality, we derive an upper bound for the objective function, thereby presenting a new and efficient approach within the RKM framework.
    \item We performed extensive experiments using four hyperspectral image datasets to evaluate the performance of our proposed model. The results indicate that our approach excels in hyperspectral image classification tasks, successfully managing the intricate spectral and spatial information present in these datasets. The results emphasize the robustness and versatility of the proposed model in various hyperspectral imaging contexts, further demonstrating its potential for practical real-world use.
    \item We carried out experiments using $38$ real-world datasets sourced from the UCI repository \cite{dua2017uci} and the KEEL repository \cite{derrac2015keel}. The outcomes of these numerical experiments, bolstered by statistical analysis, show that the proposed $R^2KM$ model reliably exceeds the performance of the baseline models.
\end{enumerate}
The rest of this paper is structured as follows: Section \ref{Related Works} reviews the relevant literature. Section \ref{$R^2KM$2} explains the mathematical framework for the proposed $R^2KM$ model. We discuss the generalization error bound of the proposed $R^2KM$ model in Section \ref{Generalization Error Bound Analysis}. Section \ref{Numerical Experiments and Results} presents the experimental results along with a statistical analysis of the proposed $R^2KM$ model. Finally, Section \ref{Conclusion} concludes the paper with future directions.
\section{Related Works}
\label{Related Works}
In this section, we discuss the mathematical formulation of RVFL and RKM model.
\subsection{Random Vector Functional Link (RVFL) Network}
The RVFL \cite{pao1994learning} model is structured as a feed-forward neural network comprising three distinct layers: an input layer, a hidden layer, and an output layer. The weights and biases connecting the input layer to the hidden layer are randomly initialized within a specified range using a uniform distribution and are kept constant during the training phase. In contrast, the weights connecting the hidden layer to the output layer—known as output weights—are determined analytically through a closed-form solution. The architecture of the RVFL model is illustrated in Fig. \ref{Geometrical structure of RVFL model}.
\begin{figure}[ht!]
    \centering
    \includegraphics[width=0.50\textwidth,height=5cm]{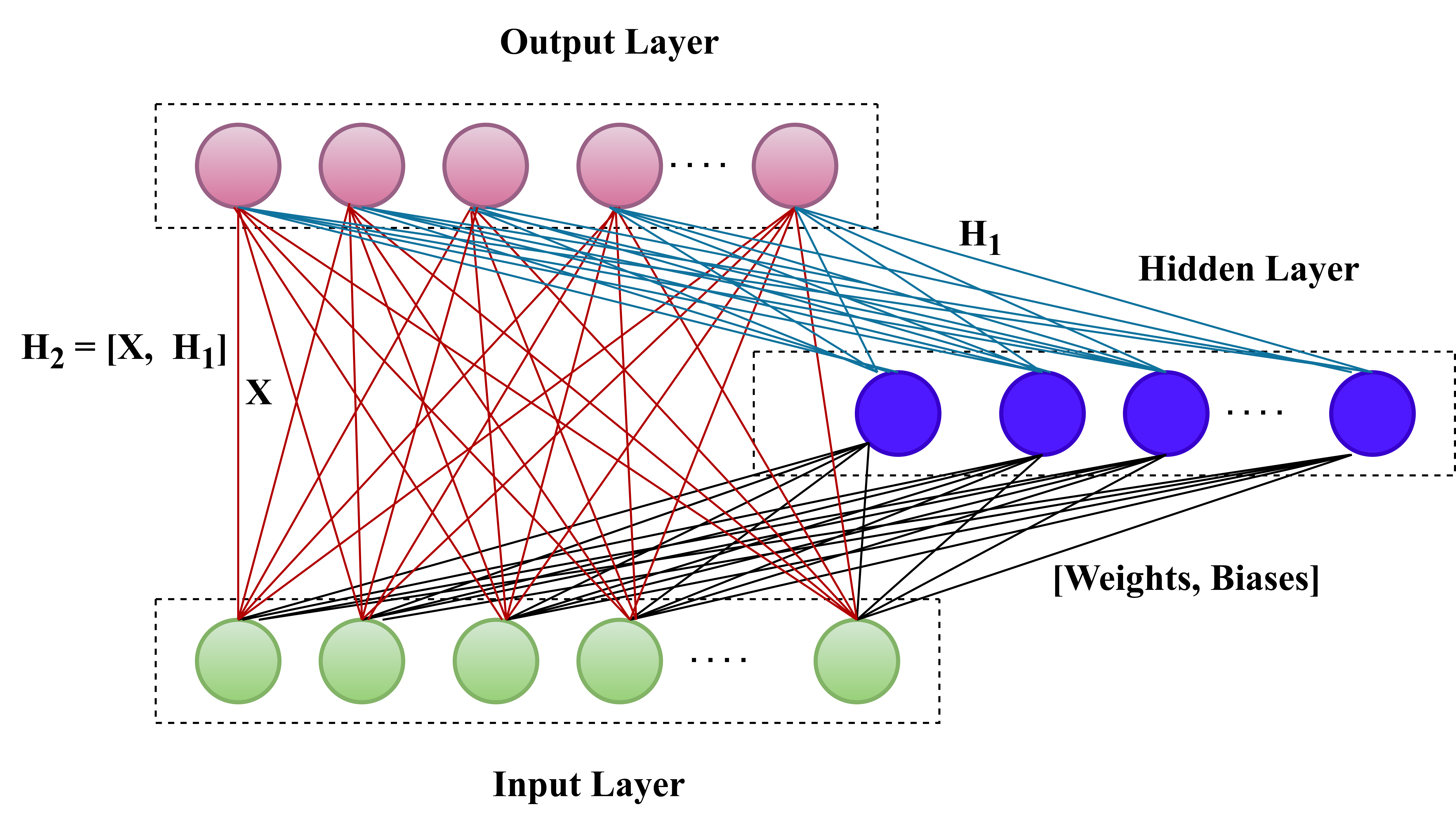}
    \caption{Geometrical structure of RVFL model}
    \label{Geometrical structure of RVFL model}
\end{figure}

Let the training dataset be $U = \{(x_i, y_i) \mid i=1, 2, \ldots, n\}$, where $x_i \in \mathbb{R}^{1 \times m}$, and $y_i \in \{+1,-1\}$ denotes the label. Define $Y = (y_1^t, y_2^t, \ldots, y_n^t)^t \in \mathbb{R}^{n \times 2}$, and $X = (x_1^t, x_2^t, \ldots, x_n^t)^t \in \mathbb{R}^{n \times m}$ as the matrices containing all target and input vector, respectively. Following the projection of the input matrix using randomly initialized weights and biases, activation function \(\phi\) is applied to produce the matrix \(H_1\) (hidden layer), expressed as follows:
\begin{align}
    H_1  = \phi(XW_1 + b_1) \in \mathbb{R}^{n \times h_l},
\end{align}
where \( W_1 \in \mathbb{R}^{m \times h_l} \) denotes the weight matrix, which is chosen randomly with values taken from a uniform distribution over \([-1, 1]\), and \( b_1 \in \mathbb{R}^{n \times h_l} \) represents the bias matrix. Therefore, \( H_1 \) is expressed as:
\begin{align}
   H_1= \begin{bmatrix}
      \phi(x_1w_1+b^{(1)}) & \ldots & \phi(x_1w_{h_l}+b^{(h_l)}) \\
     \vdots & \vdots & \vdots \\
      \phi(x_nw_n+b^{(1)}) & \ldots & \phi(x_nw_{h_l}+b^{(h_l)})
    \end{bmatrix},
\end{align}
where \( x_i \in \mathbb{R}^{1 \times m} \) represents the \( i \)-th sample in the matrix \( X \), \( w_k \in \mathbb{R}^{n \times 1} \) denotes the \( k \)-th column vector of the weight matrix \( W_1 \), and the bias term for the \( j \)-th hidden node is denoted as \( b^{(j)} \). The following matrix equation is used to compute the weights of the output layer:
\begin{align}
\label{eq:3}
    \begin{bmatrix}
            X & H_1
        \end{bmatrix}W_2 = \hat{Y}.
\end{align}
Here, \( W_2 \in \mathbb{R}^{(m + h_l) \times 2} \) represents the weight matrix that links the concatenated hidden and the input nodes to the output nodes, and \( \hat{Y} \) denotes the predicted output. The objective function derived from Eq. (\ref{eq:3}) is formulated as follows:
\begin{align}
\label{eq:4}
    (W_2)_{min} = \underset{W_2}{\arg\min} \frac{\mathcal{C}}{2}\|H_2W_2 - Y\|^2 + \frac{1}{2}\|W_2\|^2,
\end{align}
where $H_2 = \begin{bmatrix}
            X & H_1
        \end{bmatrix}$. The solution of Eq. (\ref{eq:4}) is obtained as follows:
\begin{align}
    (W_2)_{min}=\left\{\begin{array}{ll}\left({H_2}^t {H_2}+\frac{1}{\mathcal{C}} I\right)^{-1} {H_2}^{t} {Y}, & (m+h_l) \leq n,\vspace{3mm} \\ 
{H_2}^t\left({H_2} {H_2}^t+\frac{1}{\mathcal{C}} I\right)^{-1} {Y}, & n<(m+h_l), \end{array}\right.
\end{align}
where \(I\) denotes the identity matrix of appropriate dimensions, and \(\mathcal{C} > 0\) is a tuning parameter.

\subsection{Restricted Kernel Machine (RKM)}
This subsection provides an overview of the RKM model as detailed by \citet{suykens2017deep}, highlighting its close relationship with the well-established LSSVM \cite{suykens1999least} model. RKM utilizes the kernel trick to transform the data into a high-dimensional feature space, enabling the construction of a nonlinear separating hyperplane. The objective function for RKM is given as follows:
\begin{align}
\label{eq:8}
      J = & \frac{\gamma}{2} Tr(W^TW) + \sum_{i=1}^N (1-(\phi(x_i)^TW+b)y_i)h_i - \frac{\eta}{2} \sum_{i=1}^Nh_i^2,
\end{align}
where \(\gamma\) and \(\eta\) are the regularization parameters, \(b\) is the bias term, and \(h\) represents the hidden features. The solution to equation \eqref{eq:8} is obtained by taking the partial derivatives of \(J\) with respect to \(W\), \(b\), and \(h_i\), and then setting these derivatives to zero. For a detailed derivation, refer to \citet{suykens2017deep}.

\section{Proposed randomized based restricted kernel machine ($R^2KM$)}
\label{$R^2KM$2}
In this section, we first give the formulation of the proposed $R^2KM$, and then we discuss detailed mathematical formulation along with the solution of the proposed $R^2KM$ model. In $R^2KM$, the specific number of enhancement nodes and their activation functions do not need to be predefined if the kernel function is provided. This approach integrates both visible and hidden variables, making it distinct from traditional kernel methods. By using kernel functions, $R^2KM$ maps input data into a higher-dimensional feature space without explicitly specifying the hidden layer structure. This approach is similar to the energy function used in restricted Boltzmann machines (RBM) \cite{hinton2006fast}, creating a link between kernel methods and RBM. $R^2KM$ integrates the flexibility of kernel functions with shallow neural networks, providing a robust and versatile tool for identifying complex data patterns and connecting kernel methods with shallow neural networks. 

Assume that the function \( \psi: x_i \rightarrow \psi(x_i) \) transforms the training samples from the input space into a high-dimensional feature space during both the training and prediction phases. The objective function of the proposed $R^2KM$ model is defined as follows:

\begin{align}
\label{eq:1}
    \underset{\beta}{\min} f(\beta)& = \underset{\beta}{\min} \frac{\eta}{2}\|\beta\|^2 + \frac{1}{2\lambda}\sum_{i=1}^n \zeta_i^T\zeta_i \nonumber \\
    & ~s.t.~~ y_i - \beta^Tz(x_i) = \zeta_i, ~\forall i \in \{1,2,\hdots,n\},
\end{align}
where \( z(x_i) \) denotes a feature vector that comprises both the initial input feature vector and the output feature vector generated by the enhancement nodes. \( \beta \) is the weight matrix that links the hidden layer and the input layer to the output layer. $\zeta_i$ represents the error of the class samples and $\eta$ and $\lambda$ $(>0)$ are the tunable parameters, respectively.

The $R^2KM$ formulation provides an upper bound for the objective function \eqref{eq:1} and incorporates the concept of conjugate feature duality by applying the Fenchel–Young inequality \cite{rockafellar1974conjugate}, which is defined as
\begin{align}
\label{eq:10}
    \frac{1}{2\lambda}\zeta^T\zeta \geq \zeta^Th - \frac{\lambda}{2}h^Th, \hspace{0.2cm} \forall \hspace{0.1cm} \zeta, h.
\end{align}
The resulting  $R^2KM$ objective function is then given by:
\begin{align}
\label{eq:11a}
    f(\beta) \geq ~& \sum_{i=1}^n \zeta_i^T h_i - \frac{\lambda}{2} h_i^Th_i + \frac{\eta}{2} Tr(\beta^T\beta) \nonumber \\ 
   & ~s.t.~~ y_i - \beta^Tz(x_i) = \zeta_i, ~\forall i \in \{1,2,\hdots,n\}. 
\end{align}
Incorporate the constraints to derive the following tight upper bound for \( f(\beta) \):
\begin{align}
\label{eq:11}
    f(\beta) \geq ~& \sum_{i=1}^n (y_i - \beta^Tz(x_i))^T h_i - \frac{\lambda}{2} h_i^Th_i + \frac{\eta}{2} Tr(\beta^T\beta) = \hat{f}(\beta). 
\end{align}
The optimization process involves identifying the stationary points of \( f(\beta) \):
\begin{align}
    & \frac{\partial \hat{f}(\beta)}{\partial \beta} = 0 \implies y_i = \beta^Tz(x_i) + \lambda h_i, ~\forall i \in \{1,2,\hdots,n\}, \label{eq:13} \\
    & \frac{\partial \hat{f}(\beta)}{\partial h_1} = 0 \implies \beta = \frac{1}{\eta}\sum_{i=1}^{n}z(x_i)h_i^T, ~\forall i \in \{1,2,\hdots,n\}.  \label{eq:14} 
\end{align}
By eliminating the weight vector, we substitute \( \beta\) from Eq. \eqref{eq:14} into Eq. \eqref{eq:13}, we obtain:
\begin{align}
\label{eq:12}
    y_i = \frac{1}{\eta} \sum_{i=1}^{n} z(x_i)z(x_i)^T h_i  + \lambda h_i, ~\forall i \in \{1,2,\hdots,n\}.
\end{align}
The Eq. \eqref{eq:12} can be expressed in matrix form as:
\begin{align}
    Y = \frac{1}{\eta} KK^TH  + \lambda H,
\end{align}
where $H=[h_1, h_2, h_3, \ldots, h_n]^T$, $K=[X, \psi(X)]$ and $Y=[y_i, y_2,\ldots, y_n]^T$ denotes the target vector. By locating the stationary points of the objective function, we derive the following equation:
\begin{align}
\label{eqs:14}
    H = \left( \frac{1}{\eta} KK^T + \lambda I \right)^{-1} Y.
\end{align}
Then \eqref{eqs:14} can be rewritten as follows:
\begin{align}
    H = \left( \frac{1}{\eta} [X~\psi(X)] \begin{bmatrix} X^T \\ \psi(X)^T \end{bmatrix} + \lambda I \right)^{-1} Y, \nonumber \\
     H = \left( \frac{1}{\eta} (XX^T + \psi(X)\psi(X)^T) + \lambda I \right)^{-1} Y. \label{eqs:115}
\end{align}
Define kernel matrix as: $\Omega = XX^T$ be a linear kernel function and $\hat{\Omega} = \psi(X)\psi(X)^T = \mathscr{K}(X, X^T)$ be the Gaussian kernel function. Then, Eq. \eqref{eqs:115} simplifies to:
\begin{align}
\label{eqs:16}
    H = \left( \frac{1}{\eta} (\Omega + \hat{\Omega}) + \lambda I \right)^{-1} Y. 
\end{align}
After calculating the optimal values of \( H \), the new data point \( x \) can be classified as follows:
\begin{align}
\label{eqs:17}
    \hat{y} = sign\left( \frac{1}{\eta} \sum_j h_j[\mathscr{K}(x_j, x) + x_jx] \right).
\end{align}
For regression, the final output can be determined as follows:
\begin{align}
\label{eqs:18}
    \hat{y} =  \frac{1}{\eta} \sum_j h_j[\mathscr{K}(x_j, x) + x_jx].
\end{align}
Algorithm \ref{$R^2KM$ classifier} outlines the key steps involved in the proposed $R^2KM$ algorithm.

\begin{algorithm}
\caption{$R^2KM$ classifier}
\label{$R^2KM$ classifier}
\textbf{Require:} Consider the input matrix $X \in \mathbb{R}^{n \times m}$ and $Y \in \mathbb{R}^{n \times 1}$ be the target matrix. Here, \(m\) represents the number of features of the input sample and \(n\) represents the number of samples, respectively. 
\begin{algorithmic}[1]
\STATE Find the kernel matrix $\hat{\Omega} = \mathscr{K}(X, X^T)$, and $\Omega = XX^T$. \vspace{-0.4cm}
\STATE Calculate $H$ using Eqs \eqref{eqs:16}.
\STATE Finally, the output is predicted using the following Eq.: (\ref{eqs:17}) for classification and Eq. \eqref{eqs:18} for regression.
\end{algorithmic}
\end{algorithm}

\section{Generalization Error Bound Analysis}
\label{Generalization Error Bound Analysis}
Here, we theoretically examine the generalization capability of $R^2KM$ using Rademacher complexity. We start by defining Rademacher's complexity.
\begin{definition}
\label{Def1}
    The empirical Rademacher complexity of \( \mathscr{Q} \) for a function set \( \mathscr{Q} \) on \( S \) and a sample set \( S = \{x_1, \ldots, x_n\} \), which consists of \(n \) independent samples taken from the distribution \( \mathcal{D} \) is defined as follows:
    \begin{align}
    \hat{\mathcal{I}}_n({\mathscr{Q}}) & = \mathbb{E}_{\omega}\left[ \underset{q \in \mathscr{Q}}{\sup} | 2/n \sum_{i=1}^n \omega_i q(x_i) |: x_1, x_2,\ldots, x_n \right],
    \end{align}
    where independent Rademacher random variables with values in \(\{+1, -1\}\) are represented by \( \omega = (\omega_1, \ldots, \omega_n) \). Then, we define \(\mathscr{Q}\)'s Rademacher complexity as follows:
    \begin{align}
    I_n(\mathscr{Q}) & = \hat{\mathbb{E}}_S[\hat{I}_n(\mathscr{Q})] =\hat{\mathbb{E}}_{S\omega}\left[ \underset{q \in \mathscr{Q}}{\sup} | 2/n \sum_{i=1}^n \omega_i q(x_i) |\right].
\end{align}
\end{definition}
Based on the definition and lemmas of Rademacher complexity provided in \cite{bartlett2002rademacher}, we present the generalization error bound for $R^2KM$ in the following theorem.
\begin{theorem}
    Let \( \kappa \in (0, 1) \), \( N \in \mathbb{R}^+ \), and assume we have dataset \( U = \{(x_i, y_i)\}_{i=1}^n \), where each sample is drawn independently and identically according to a specific probability distribution \( \mathcal{D} \), with \( y_i \in \{-1, +1\} \) representing binary class labels. Define the class function \( \mathscr{Q}_N = \{q \mid q: x \rightarrow h^T\phi(x), \|h\| \le N\} \) and \( \hat{\mathscr{Q}}_N = \{\hat{q} \mid \hat{q}: x \rightarrow h^T\phi(x), \|h\| \le N\} \). Then, with a confidence level of at least \( 1 - \epsilon \) for the dataset \( U \), every function \( q_m(x) \in \mathscr{Q}_N \) adheres to the following condition:
    \begin{align}
\label{Th1}
    \mathbb{E}_{\mathcal{D}}[q_m(x)] \le \frac{1}{n} \sum_{i=1}^n \xi_i + 3 \sqrt{\frac{\ln(2/\epsilon)}{2n}}
    + \frac{4N}{n}\sqrt{\sum_{i=1}^n \psi(x_i, x_i)}.
\end{align}
\end{theorem}
\begin{proof}
Introduce the Heaviside function, which is defined as:
\begin{align}
  \mathcal{H}(x)=\left\{\begin{array}{ll}{0}, & \text{if}~~ x \leq 0,   \\     {1}, & \text{if}~~ x > 0.
     \end{array}\right. 
\end{align}
It then follows that:
\begin{align}
    \mathcal{P}_{\mathcal{D}} (yq(x) \leq 0)  = \mathbb{E}_{\mathcal{D}}\left[\mathcal{H}(-yq(x) \right],
\end{align}
where $\mathbb{E} = \mathbb{E}_{\mathcal{D}}$ denotes the true expectation over \( \mathcal{D} \). Let \( \Lambda: \mathbb{R} \to [0, 1] \) represent a loss function defined as follows:
\begin{align}
    \Lambda=\left\{\begin{array}{lll}{1}, & \text{if}~~ 0 < x, \\ {x+1} , & \text{if}~~ -1 \leq x \leq 0, \\
    {0}, & \text{otherwise}.
 \end{array}\right. 
\end{align}
Since the function \( \Lambda(\cdot) \) dominates \( \mathcal{H}(\cdot) \) on the support of \( \mathcal{D} \), we obtain the following inequality by applying Theorem (7) from \cite{bartlett2002rademacher}: 
\begin{align}
   & \mathbb{E}_{\mathcal{D}}\left(\mathcal{H}(\hat{q}(x, y))-1 \right) \leq \mathbb{E}_{\mathcal{D}}\left(\Lambda\mathbb(\hat{q}(x, y))-1 \right)  \leq \hat{\mathbb{E}}\left(\Lambda\mathbb(\hat{q}(x, y))-1 \right) + \hat{I}_n\left( (\Lambda - 1) \circ \hat{\mathscr{Q}} \right) + 3\sqrt{\frac{\ln(2/\epsilon)}{2n}}.
\end{align}
Hence,
\begin{align}
\label{eq:26}
   & \mathbb{E}_{\mathcal{D}}\left(\mathcal{H}(\hat{q}(x, y)) \right) \leq \mathbb{E}_{\mathcal{D}}\left(\Lambda\mathbb(\hat{q}(x, y)) \right) \leq \hat{\mathbb{E}}\left(\Lambda\mathbb(\hat{q}(x, y)) \right) + \hat{I}_n\left( (\Lambda - 1) \circ \hat{\mathscr{Q}} \right) + 3\sqrt{\frac{\ln(2/\epsilon)}{2n}}.
\end{align}
Given the constraints of the $R^2KM$ model with the optimal \( h \), we achieve:
\begin{align}
\label{eq:27}
    \mathbb{E}_{\mathcal{D}}\left(\Lambda\mathbb(\hat{q}(x, y)) \right) & \leq \frac{1}{n} \sum_{i=1}^n [1 -  y_iq(x_i)]_+ \nonumber \\
    & \leq \frac{1}{n} \sum_{i=1}^n [y_i\xi_i]_+ \nonumber \\
    & \leq \frac{1}{n} \sum_{i=1}^n \xi_i.
\end{align}
According to Theorem (14) in \cite{bartlett2002rademacher}, and given that the Lipschitz function \( \Lambda(\cdot) \) has a Lipschitz constant of $1$ with \( \Lambda(0) = 0 \), we have:
\begin{align}
    \hat{\mathcal{I}}_n \left( (\Lambda - 1) \circ \hat{\mathscr{Q}} \right) \leq 2 \hat{\mathcal{I}}_n ( \hat{\mathscr{Q}} ).
\end{align}
According to def. \eqref{Def1}, with \( y \in \{ -1, +1 \} \), we have:
\begin{align}
     \hat{\mathcal{I}}_n (\hat{\mathscr{Q}}) & = \mathbb{E}_{\omega}\left[ \underset{\hat{q} \in \hat{\mathscr{Q}}}{\sup} | 2/n \sum_{i=1}^n \omega_i \hat{q}(x_i, y_i) | \right] \nonumber \\
     & = \mathbb{E}_{\omega}\left[ \underset{q \in \mathscr{Q}}{\sup} | 2/n \sum_{i=1}^n \omega_i y_ig(x_i) | \right]  \nonumber \\
     & = \mathbb{E}_{\omega}\left[ \underset{q \in \mathscr{Q}}{\sup} | 2/n \sum_{i=1}^n \omega_i q(x_i) | \right]  \nonumber \\
     & = \hat{\mathcal{I}}_n(\mathscr{Q}).
\end{align}
As stated in Lemma (22) of \cite{bartlett2002rademacher}, the empirical Rademacher complexity associated with the function class \( \mathscr{Q} \) is defined as follows:
\begin{align}
\label{eq:30}
    \hat{I}_n (\mathscr{Q}) \leq \frac{4N}{n}\sqrt{\sum_{i=1}^n \psi(x_i, x_i)}.
\end{align}
By putting Eqs. \eqref{eq:27}–\eqref{eq:30} into Eq. \eqref{eq:26}, we can draw the conclusions outlined in this theorem.
\end{proof}
Following the $R^2KM$ model outlined in \eqref{eqs:17}, we define the classification error function \( q(x) \) and derive a margin-based estimate for the misclassification probability. This is achieved by integrating the empirical expectation of \( \hat{\mathscr{Q}} \) with the empirical Rademacher complexity linked to \( \mathscr{Q} \). As \(n \) increases sufficiently, $R^2KM$ provides a strong generalization error bound for classification. The training error decreases, leading to a reduction in the generalization error as well. This theoretical finding guarantees that $R^2KM$ demonstrates enhanced generalization performance.

\section{Numerical Experiments and Results}
\label{Numerical Experiments and Results}
In this section, we conduct the experiment of the proposed $R^2KM$ model and the existing model using datasets derived from hyperspectral images. Moreover, we assess our proposed model by utilizing publicly accessible benchmark datasets from UCI and KEEL for classification and regression tasks. We compare our proposed $R^2KM$ model with ELM or RVFL without direct link (RVFLwoDL) \cite{huang2006extreme}, RVFL \cite{pao1994learning}, NF-RVFL \cite{sajid2024neuro}, and RKM \cite{suykens2017deep} models.

\subsection{Experimental Setup}
The experimental setup consists of a personal computer powered by an Intel(R) Xeon(R) Gold 6226R CPU, operating at a frequency of $2.90$ GHz and equipped with $128$ GB of RAM. The system operates on the Windows 11 platform and performs tasks using Python version $3.11$. The dataset is split randomly into two subsets, allocating $70\%$ for training and reserving $30\%$ for testing. The Gaussian kernel used is defined as \( \mathscr{K}(x_i, x_j) = e^{-\frac{1}{2\sigma^2} \| x_i - x_j \|^2} \), with the parameter \(\sigma\) selected from \(\{2^{-5}, 2^{-4}, \ldots, 2^5\}\). We utilize a five-fold cross-validation technique combined with a grid search method to optimize the hyperparameters of the models within specified ranges: $\eta = \lambda = \{10^{i} |~ i=-5, \ldots, 5 \}$.  For the RVFL, RVFLwoDL, and NF-RVFL models, we tune all hyperparameters from the range \(\{10^{-5}, 10^{-4}, \ldots, 10^5\}\), and select the number of hidden nodes from \(3\) to \(203\) in steps of \(20\). Additionally, we used nine different activation functions, indexed as follows: 1) SELU, 2) ReLU, 3) Sigmoid, 4) Sine, 5) Hardlim, 6) Tribas, 7) Radbas, 8) Sign, and 9) Leaky ReLU. 

\subsection{Experiments on Hyperspectral Image Datasets}
We conduct experiments using four widely used hyperspectral image datasets\footnote{\url{https://www.ehu.eus/ccwintco/index.php?title=Hyperspectral_Remote_Sensing_Scenes}}: the Indian Pines dataset, the Salinas dataset, the KSC dataset, and the Pavia University dataset. 
\begin{table}[ht!]
\centering
    \caption{Description of the Indian Pines dataset.}
    \label{Indian Pines dataset description}
    \resizebox{0.8\linewidth}{!}{
\begin{tabular}{clccc}
\hline No.  & Class  & Numbers of samples  & Training-set  & Testing-set  \\
\hline 
1 & Grass-trees  & 730 & 100 & 630 \\
2 &  Grass-pasture  & 483 & 100 & 383 \\
3 & Corn-notill  & 1428 & 100 & 1328 \\
4 &  Alfalfa  & 46 & 23 & 23 \\
5 &  Oats  & 20 & 10 & 10 \\
6 &  Hay-windrowed & 478 & 100 & 378 \\
7 &  Soybean-notill & 972 & 100 & 872 \\
8 & Soybean-mintill  & 2455 & 100 & 2355 \\
9 &  Corn & 237 & 100 & 137 \\
10 &  Corn-mintill  & 830 & 100 & 730 \\
11 & Soybean-clean  & 593 & 100 & 493 \\
12 &  Grass-pasture-mowed  & 28 & 14 & 14 \\
13 &  Stone-Steel-Towers  & 93 & 47 & 46 \\
14 & Buildings-Grass-Trees-Drives & 386 & 100 & 286 \\
15 &  Woods  & 1265 & 100 & 1165 \\
16 &  Wheat & 205 & 100 & 105 \\
\hline
\end{tabular}}
\end{table}
The dataset from Indian Pines was acquired using an AVIRIS spectrometer situated in the Indian Pines area of western North Indiana. This dataset comprises \(220\) spectral channels in total and a grid layout of \(145 \times 145\) pixels. The spectral range of the dataset extends from \(0.4\) to \SI{2.5}{\micro\metre}, with a spatial resolution of \SI{20}{\meter}. After excluding \(20\) bands affected by water absorption, \(200\) bands were retained for the purposes of training, testing, and validation. Table \ref{Indian Pines dataset description} provides an overview of the quantities of training and testing samples available in the dataset. With the remaining samples designated as the test set, $100$ samples are randomly selected from each feature category. Fig. \ref{Indian pines} presents a pseudo-color image of the Indian Pines dataset, accompanied by the corresponding ground truth image.

\begin{figure}[ht!]
\begin{minipage}{.45\linewidth}
\centering
\subfloat[Colour map]{\includegraphics[scale=0.89]{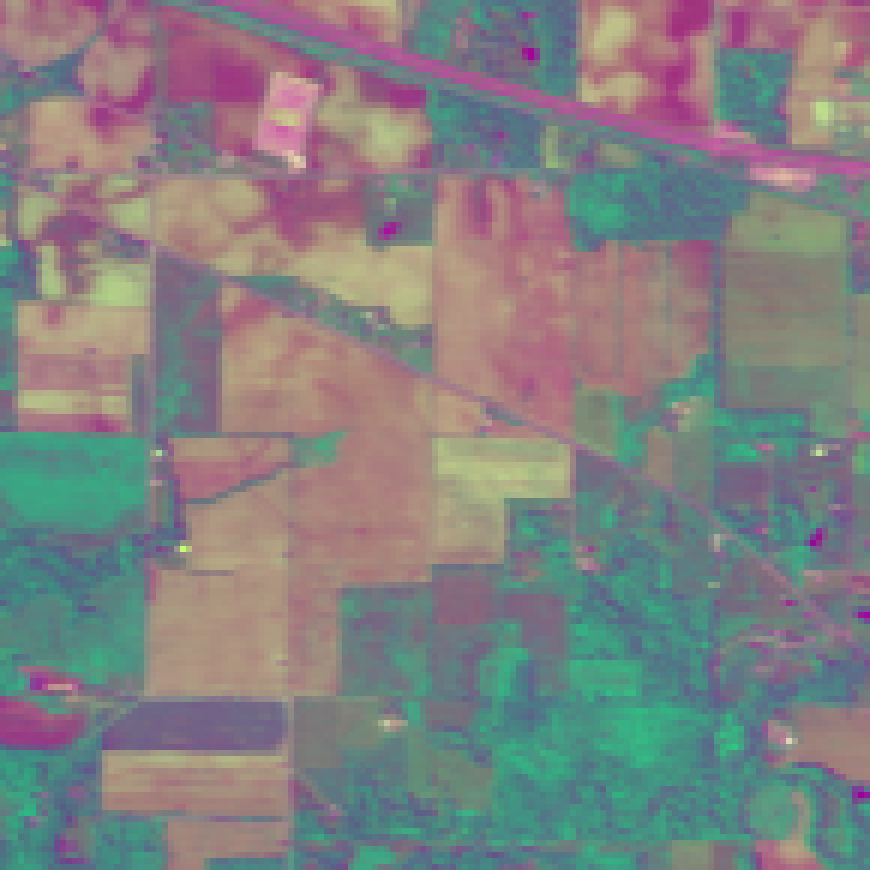}}
\end{minipage}
\begin{minipage}{.40\linewidth}
\centering
\subfloat[Ground truth]{\includegraphics[scale=0.17]{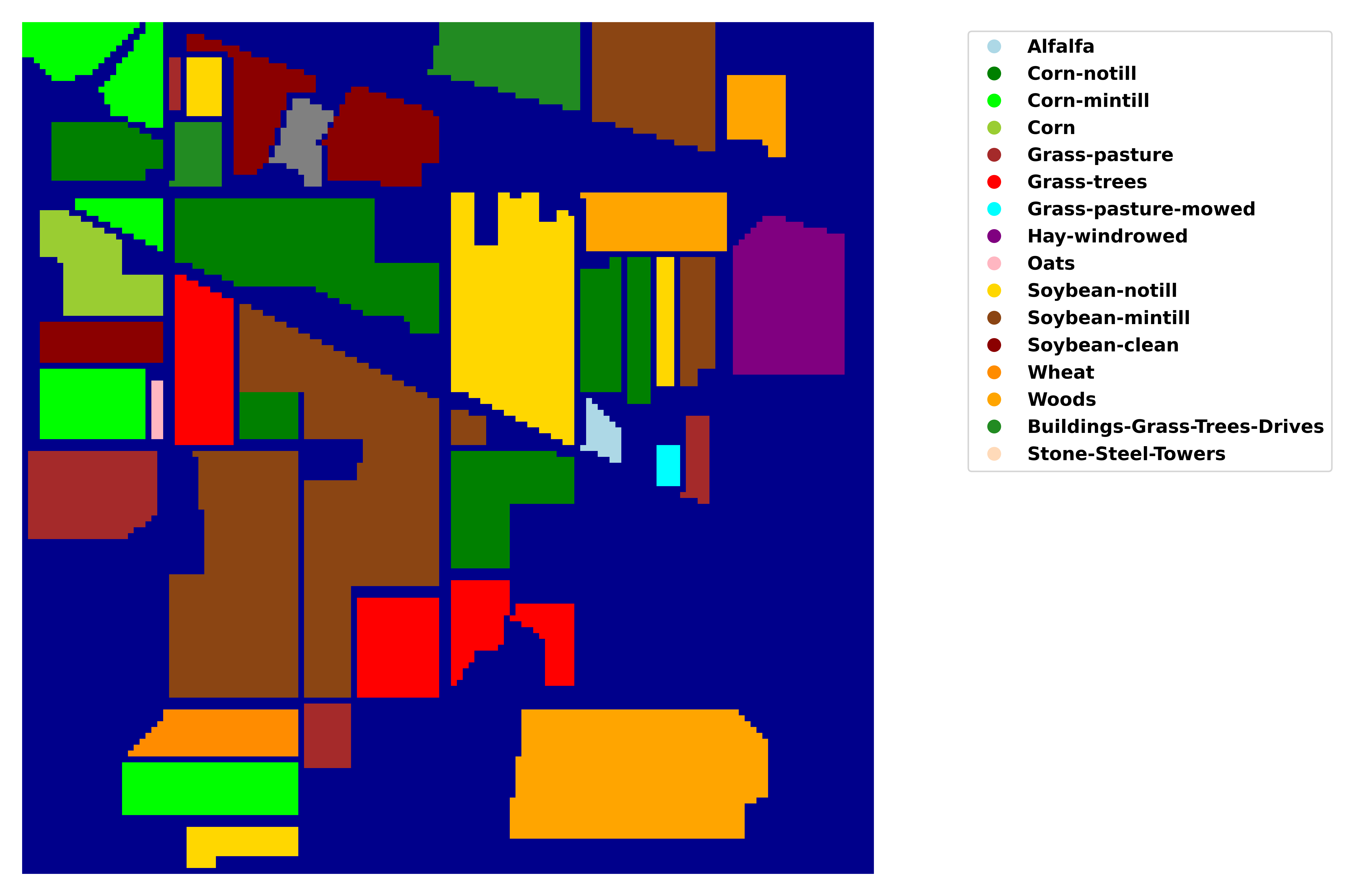}}
\end{minipage}
\caption{Indian Pines}
\label{Indian pines}
\end{figure}

\begin{table}[ht!]
\centering
    \caption{Description of the University of Pavia dataset.}
    \label{University of Pavia dataset description}
    \resizebox{0.8\linewidth}{!}{
\begin{tabular}{clccc}
\hline No.  &  Class  &  Numbers of samples  &  Training  & Testing  \\
\hline 1 &  Asphalt & 6631 & 100 & 6531 \\
2 &  Meadows  & 18649 & 100 & 18549 \\
3 &  Gravel  & 2099 & 100 & 1999 \\
4 &  Trees  & 3064 & 100 & 2964 \\
5 &  Sheets  & 1345 & 100 & 1245 \\
6 & Bare Soil  & 5029 & 100 & 4929 \\
7 &  Bitumen  & 1330 & 100 & 1230 \\
8 &  Bricks  & 3682 & 100 & 3582 \\
9 &  Shadows  & 947 & 100 & 847 \\
\hline
\end{tabular}}
\end{table}
The Pavia University image dataset was obtained using the ROSIS-$03$ optical sensor, which was deployed to map the urban landscape surrounding the University of Pavia. The scene has dimensions of $610 \times 340$ pixels and is characterized by a high spatial resolution of \SI{1.3}{\meter}. The Pavia University dataset comprised a total of $115$ spectral bands. After removing $12$ noisy bands, the remaining $103$ spectral bands are used in the experiment. Table \ref{University of Pavia dataset description} presents the distribution of training and test samples within the dataset. For each feature category, $100$ samples are randomly designated for the training set, with the leftover samples reserved for testing purposes. Fig. \ref{Pavia of University} displays both the ground truth image and the pseudo-color representation of the Pavia University dataset.

\begin{figure}[ht!]
\begin{minipage}{.45\linewidth}
\centering
\subfloat[Colour map]{\includegraphics[scale=0.40]{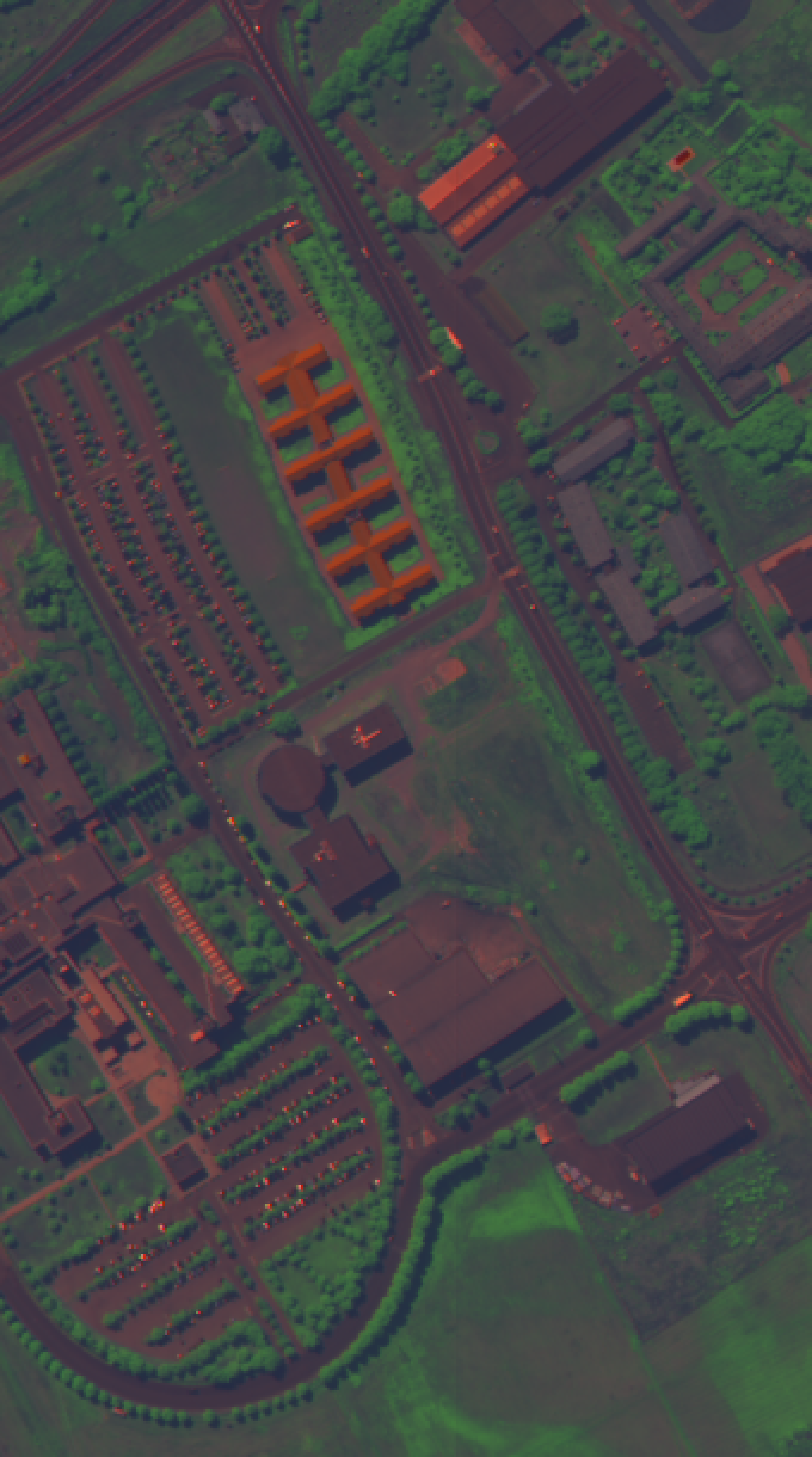}}
\end{minipage}
\begin{minipage}{.40\linewidth}
\centering
\subfloat[Ground truth]{\includegraphics[scale=0.31]{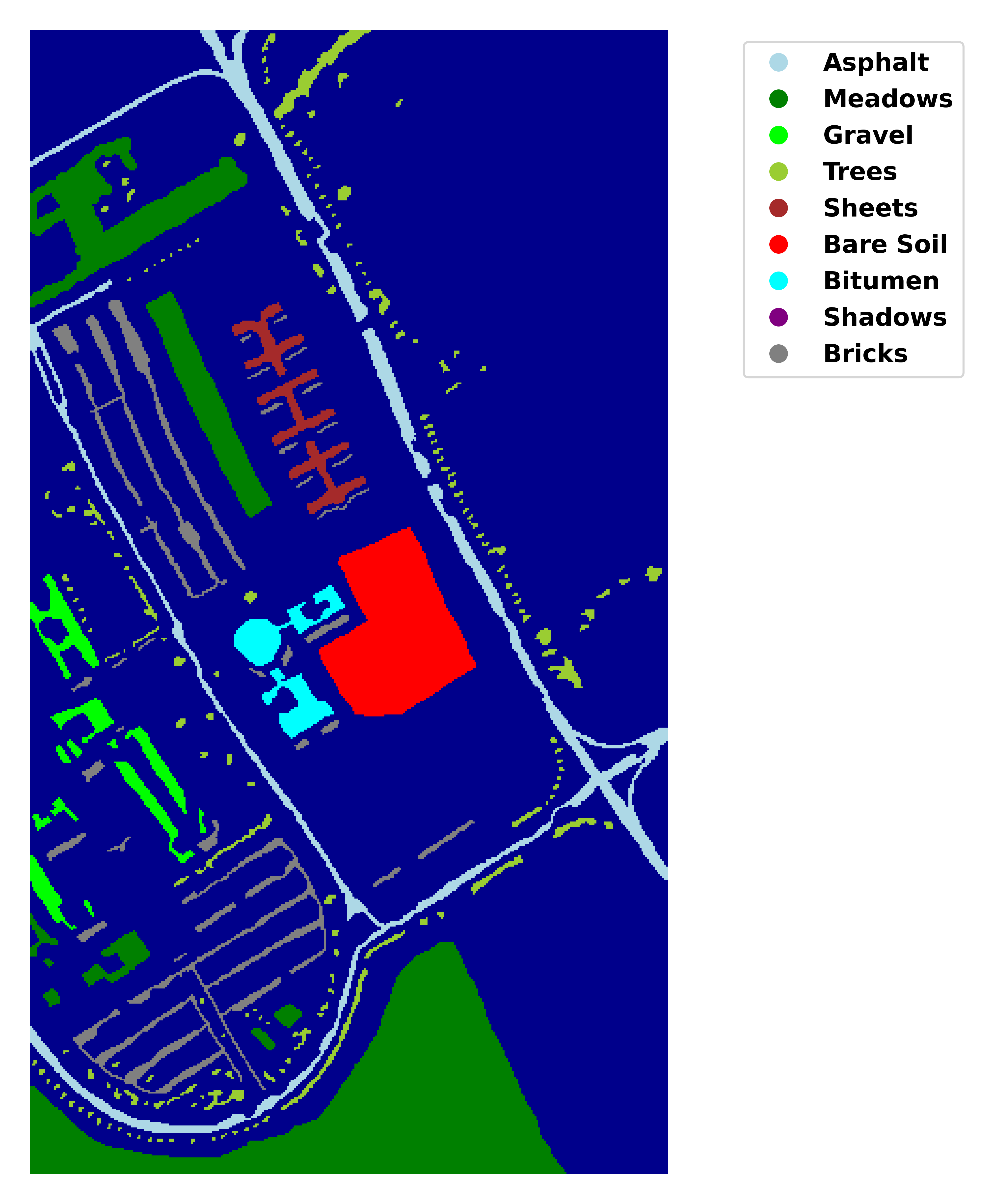}}
\end{minipage}
\caption{Pavia of University}
\label{Pavia of University}
\end{figure}

\begin{table}[htp]
\centering
    \caption{Description of the Salinas dataset.}
    \label{Salinas dataset description}
    \resizebox{0.8\linewidth}{!}{
\begin{tabular}{clccc}
\hline  No.  &  Class  & Numbers of samples  &  Training  &  Testing \\ \hline
1 & Broccoli\_green\_weeds\_2 & 3726 & 100 & 3626 \\
2 & Fallow\_rough\_plough & 1394 & 100 & 1294 \\
3 & Broccoli\_green\_weeds\_1 & 2009 & 100 & 1909 \\
4 & Fallow & 1976 & 100 & 1876 \\
5 & Stubble & 3959 & 100 & 3859 \\
6 & Fallow\_smooth & 2678 & 100 & 2578 \\
7 & Grapes\_untrained & 11271 & 100 & 11171 \\
8 & Celery & 3579 & 100 & 3479 \\
9 & Soil\_vineyard\_develop & 6203 & 100 & 6103 \\
10 & Corn\_senesced\_green\_weeds & 3278 & 100 & 3178 \\
11 & Lettuce\_romaine\_4wk & 1068 & 100 & 968 \\
12 & Lettuce\_romaine\_5wk & 1927 & 100 & 1827 \\
13 & Lettuce\_romaine\_6wk & 916 & 100 & 816 \\
14 & Lettuce\_romaine\_7wk & 1070 & 100 & 970 \\
15 & Vineyard\_untrained & 7268 & 100 & 7168 \\
16 & Vineyard\_vertical\_trellis & 1807 & 100 & 1707 \\
\hline
\end{tabular}}
\end{table}
Collected in 1998 using the AVIRIS sensor, the Salinas dataset features images captured in the Salinas Valley, California. Each image has dimensions of $512 \times 217$ pixels, boasting a high spatial resolution of \SI{3.7}{\meter} per pixel. The Salinas dataset contains $224$ spectral bands, of which $204$ were retained for our experiment after excluding $20$ bands associated with water absorption. Additionally, the dataset includes $16$ feature categories. Table \ref{Salinas dataset description} outlines the distribution of training and testing samples within the dataset. For each feature category, $100$ samples are randomly chosen for the training set, while the rest are designated for testing. Fig. \ref{Salinas} illustrates both the pseudo-color image and the ground truth image of the Salinas dataset.

\begin{figure}[ht!]
\begin{minipage}{.40\linewidth}
\centering
\subfloat[Colour map]{\includegraphics[scale=0.44]{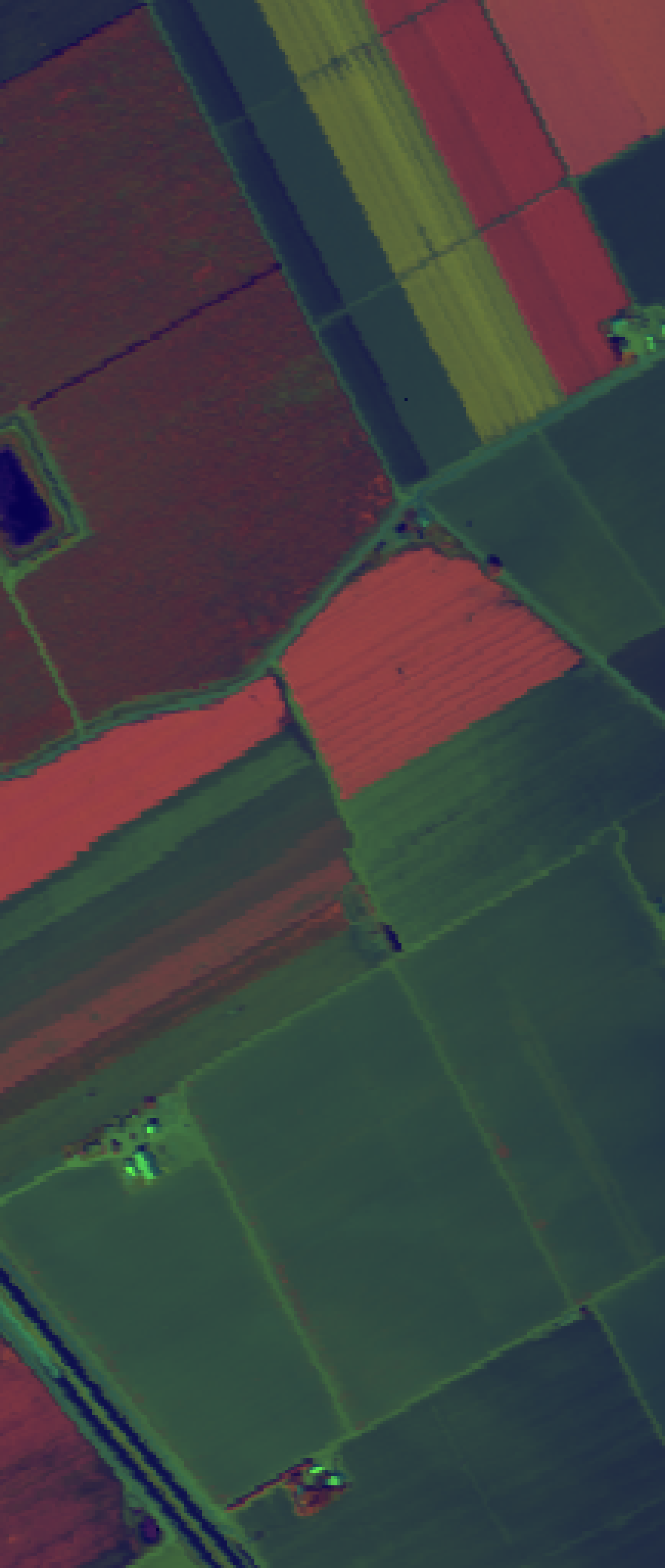}}
\end{minipage}
\begin{minipage}{.45\linewidth}
\centering
\subfloat[Ground truth]{\includegraphics[scale=0.29]{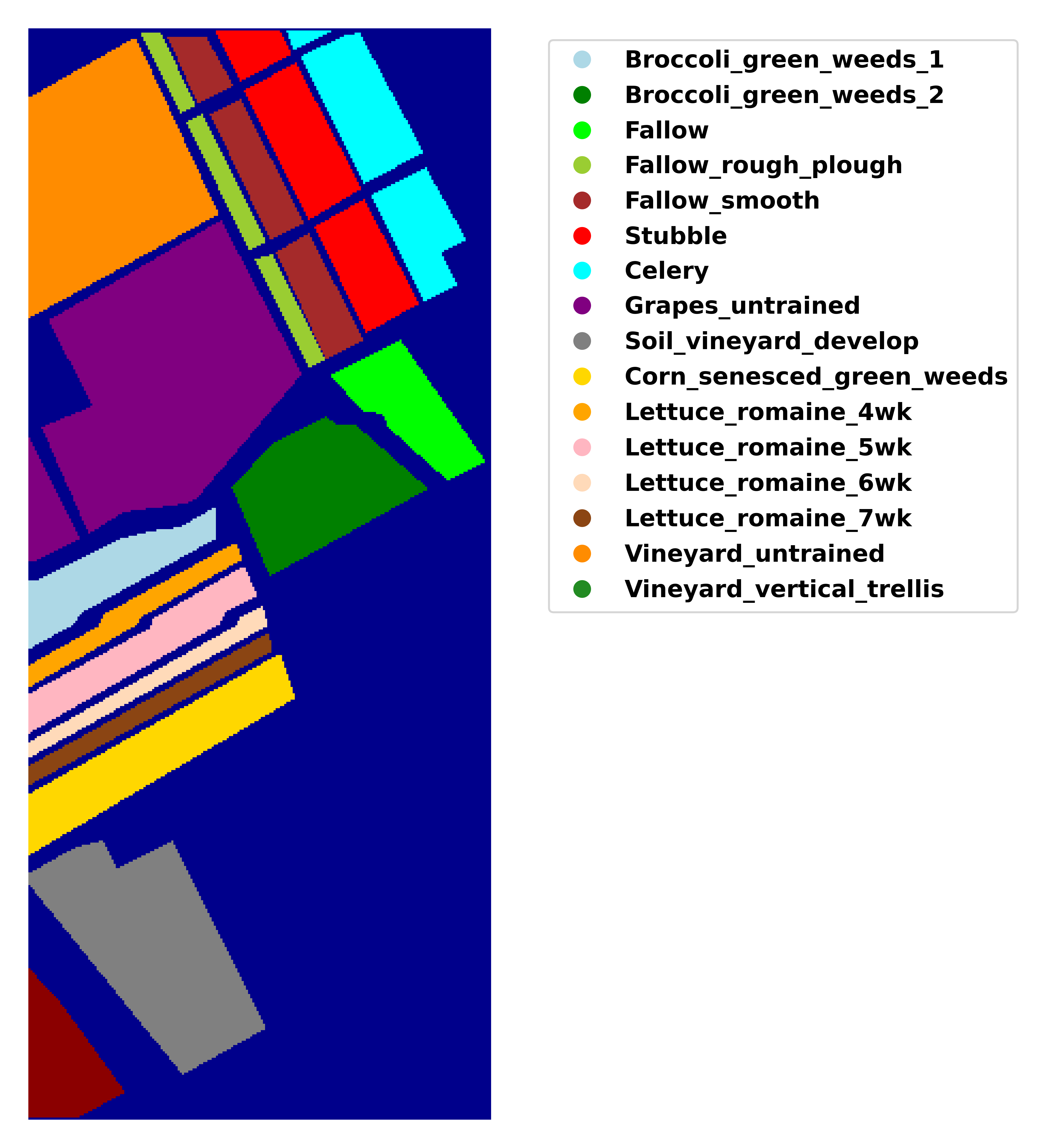}}
\end{minipage}
\caption{Salinas}
\label{Salinas}
\end{figure}

\begin{table}[htp]
\centering
    \caption{Description of the KSC dataset.}
    \label{KSC dataset description}
    \resizebox{0.85\linewidth}{!}{
\begin{tabular}{clccc}
\hline  No.  &  Class  & Numbers of samples  &  Training  &  Testing \\ \hline
1 & Water & 836 & 20 & 816 \\ 
2 & Mud flats & 1847 & 20 & 1827 \\
3 & Salt marsh & 988 & 20 & 968 \\
4 & Cattail marsh & 3198 & 20 & 3178 \\
5 & Spartina marsh & 6123 & 20 & 6103 \\
6 & Graminoid marsh & 11191 & 20 & 11171 \\
7 & Hardwood swamp & 3499 & 20 & 3479 \\
8 & Oak/Broadleaf & 3879 & 20 & 3859 \\
9 & Slash pine & 2598 & 20 & 2578 \\
10 & CP/Oak & 1314 & 20 & 1294 \\
11 & CP hammock & 1896 & 20 & 1876 \\
12 & Willow swamp & 3646 & 20 & 3626 \\
13 & Scrub & 1929 & 20 & 1909 \\
\hline
\end{tabular}}
\end{table}
The KSC dataset comprises hyperspectral remote sensing images taken over the Kennedy Space Center located in Florida, USA. The remote sensing image has a ground spatial resolution of \SI{18}{\meter} and measures $614 \times 512$ pixels. It covers a spectral range from $400$ to \SI{2500}{n\meter} and consists of $224$ bands, with $176$ bands having been pre-processed for use in experimental classification studies. The dataset comprises $13$ classes of identified ground objects, as outlined in Table \ref{KSC dataset description}. The pseudo-color image and the ground truth image are shown in Fig. \ref{KSC}.

\begin{figure}[ht!]
\begin{minipage}{.45\linewidth}
\centering
\subfloat[Colour map]{\includegraphics[scale=0.22]{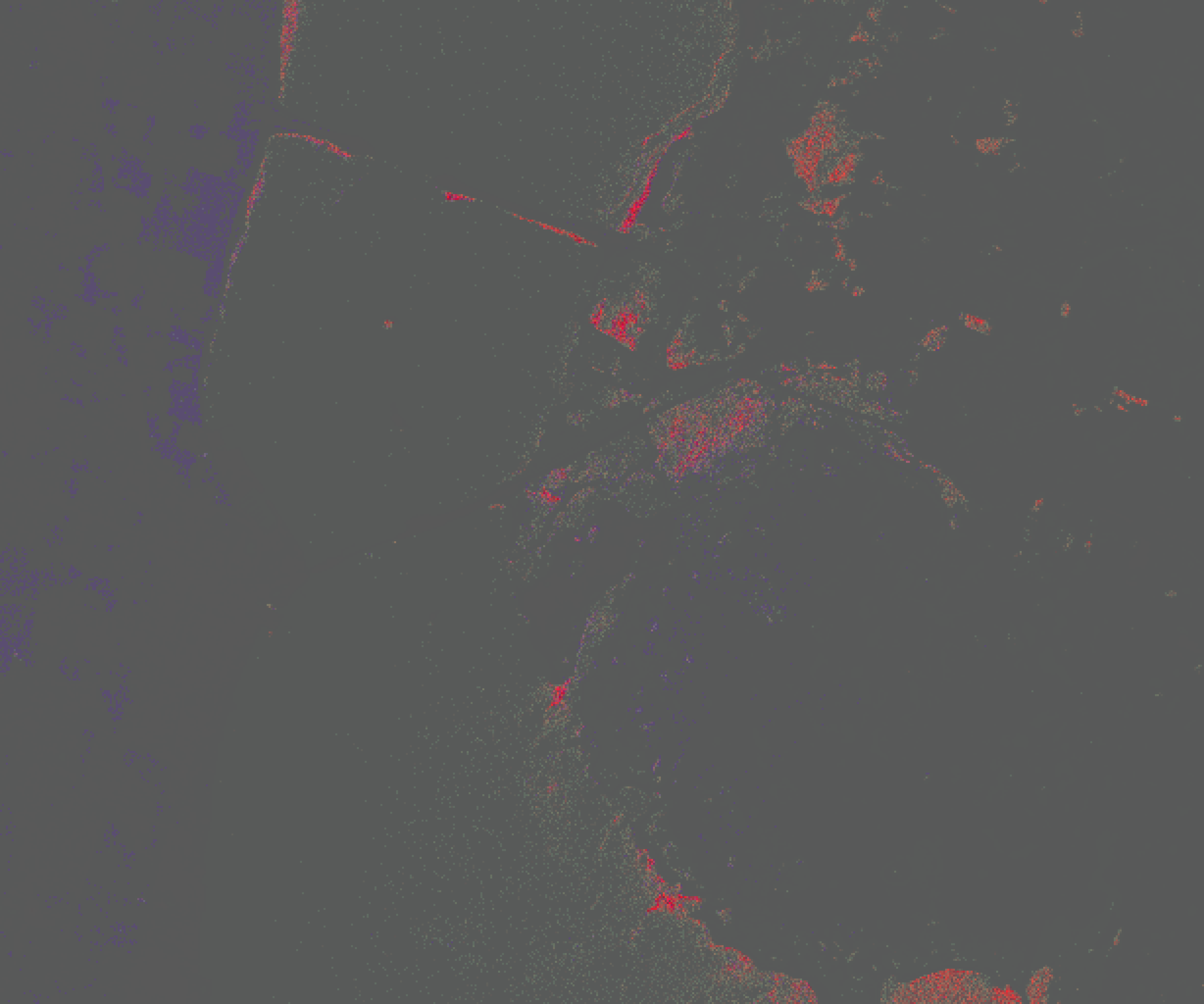}}
\end{minipage}
\begin{minipage}{.40\linewidth}
\centering
\subfloat[Ground truth]{\includegraphics[scale=0.15]{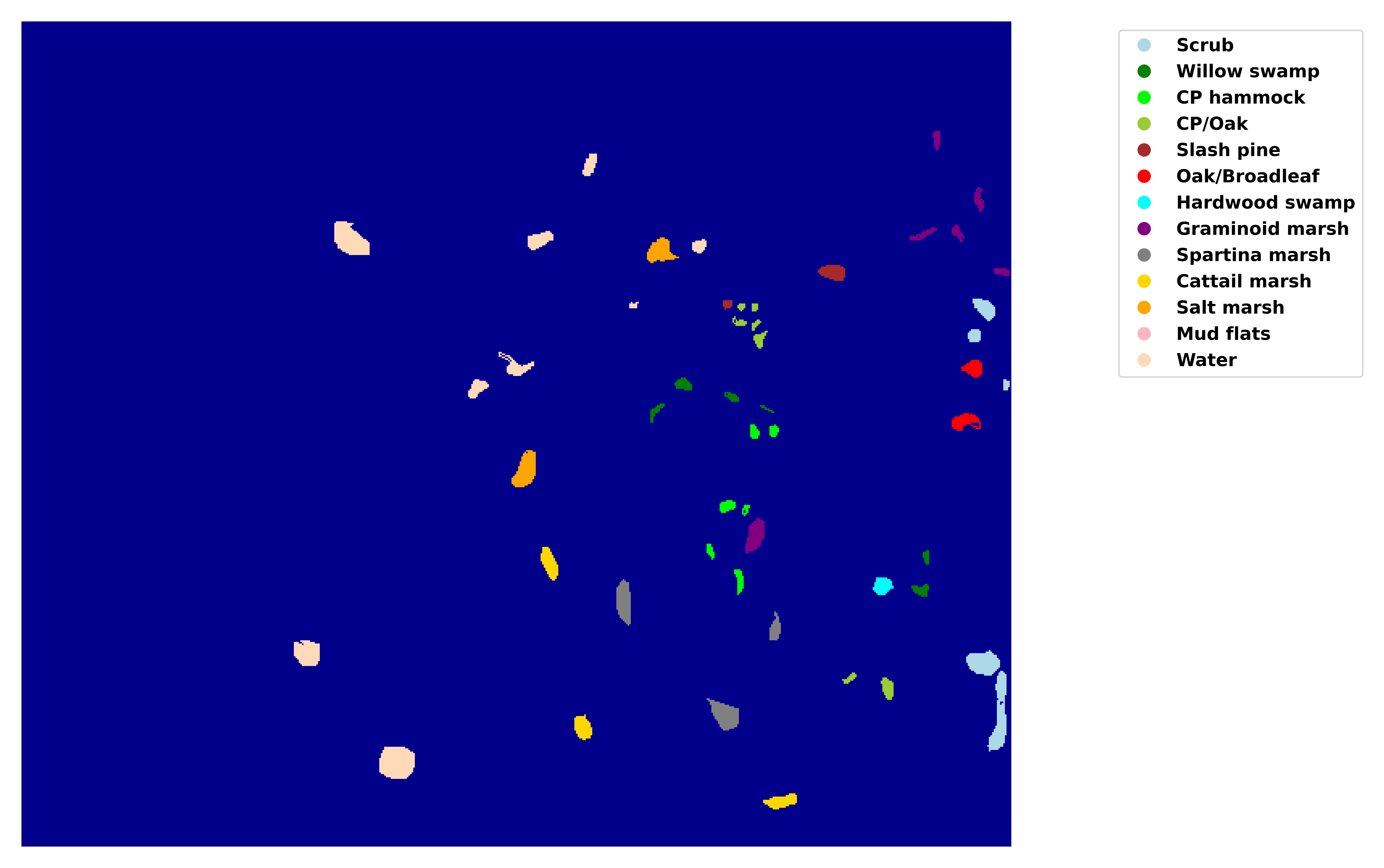}}
\end{minipage}
\caption{KSC}
\label{KSC}
\end{figure}

\begin{figure*}[ht!]
\begin{minipage}{.16\linewidth}
\centering
\subfloat[Ground truth]{\includegraphics[scale=0.60]{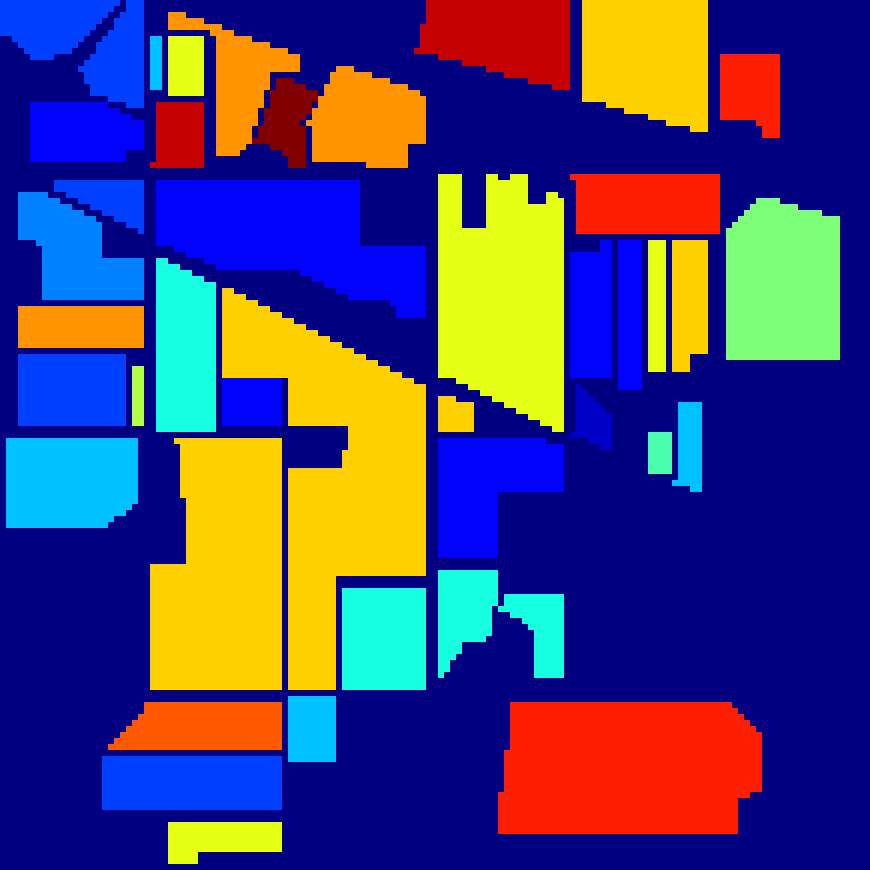}}
\end{minipage}
\begin{minipage}{.16\linewidth}
\centering
\subfloat[RVFL]{\includegraphics[scale=0.60]{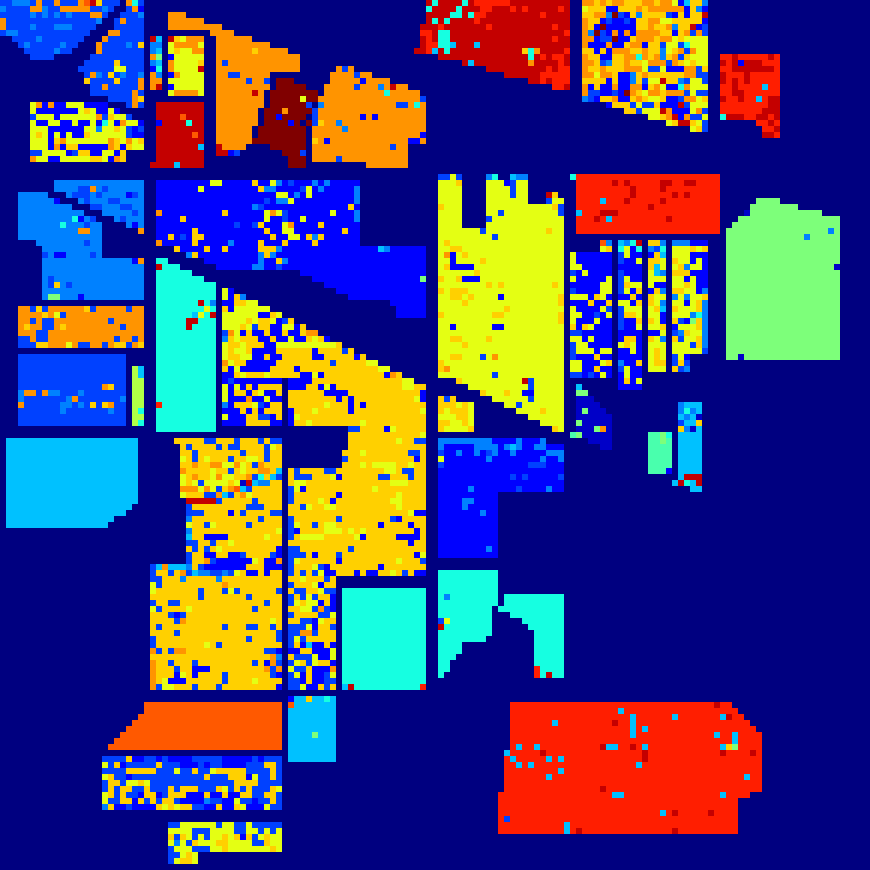}}
\end{minipage}
\begin{minipage}{.16\linewidth}
\centering
\subfloat[RVFLwoDL]{\includegraphics[scale=0.60]{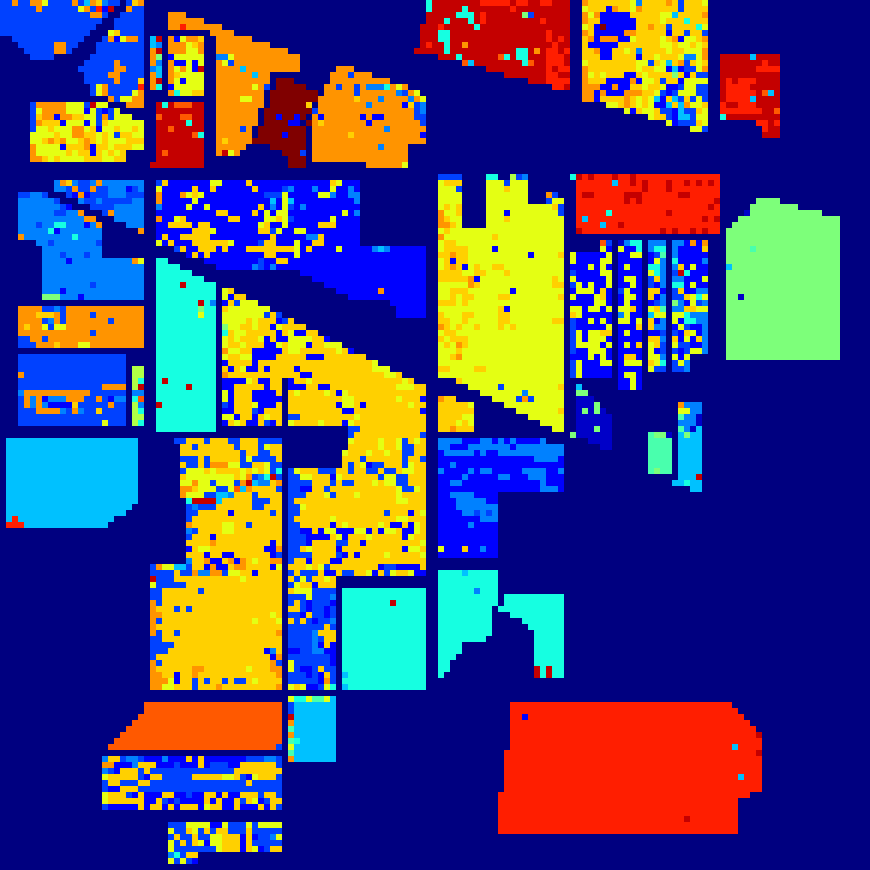}}
\end{minipage}
\begin{minipage}{.16\linewidth}
\centering
\subfloat[NF-RVFL]{\includegraphics[scale=0.60]{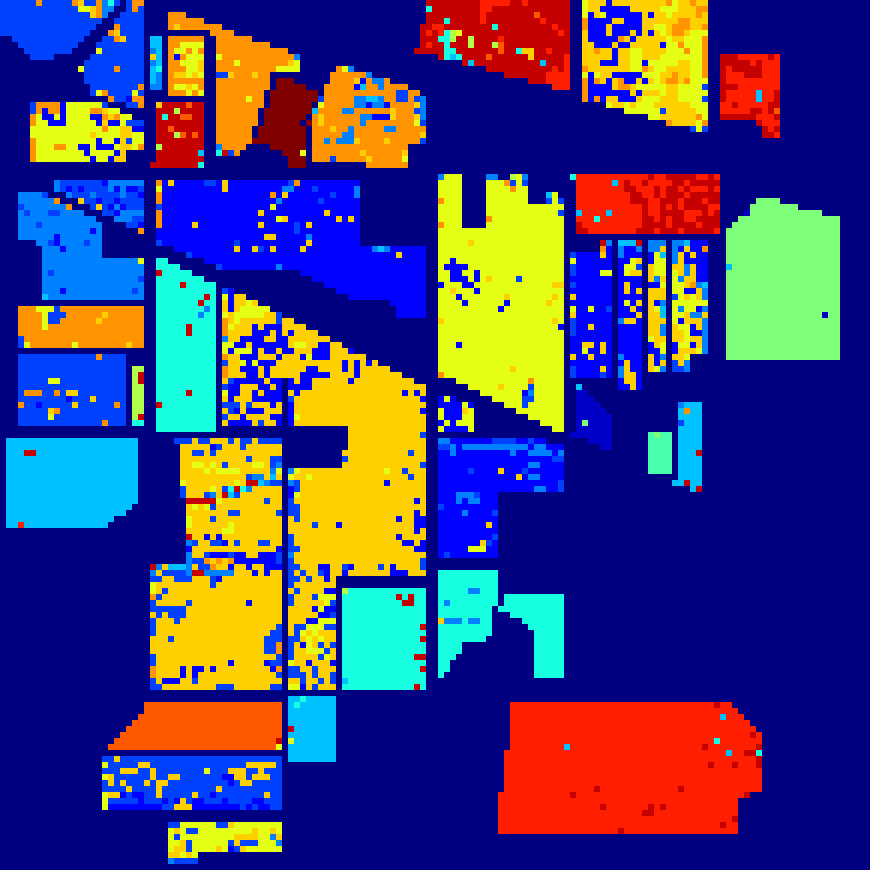}}
\end{minipage}
\begin{minipage}{.16\linewidth}
\centering
\subfloat[RKM]{\includegraphics[scale=0.60]{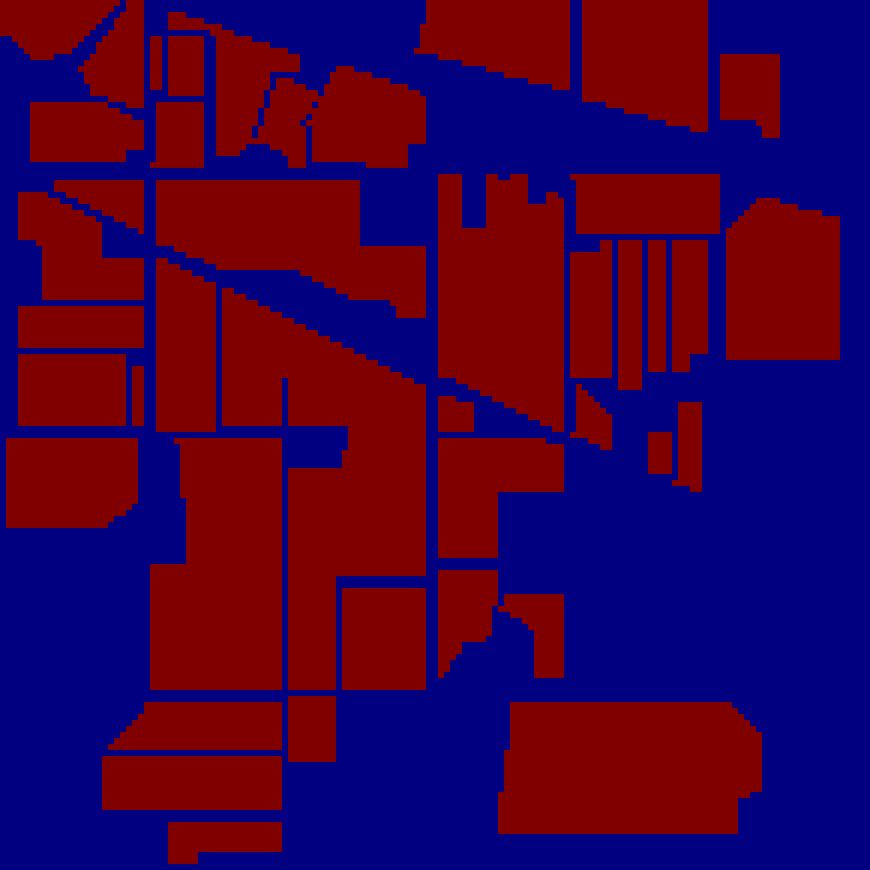}}
\end{minipage}
\begin{minipage}{.16\linewidth}
\centering
\subfloat[$R^2KM$]{\includegraphics[scale=0.60]{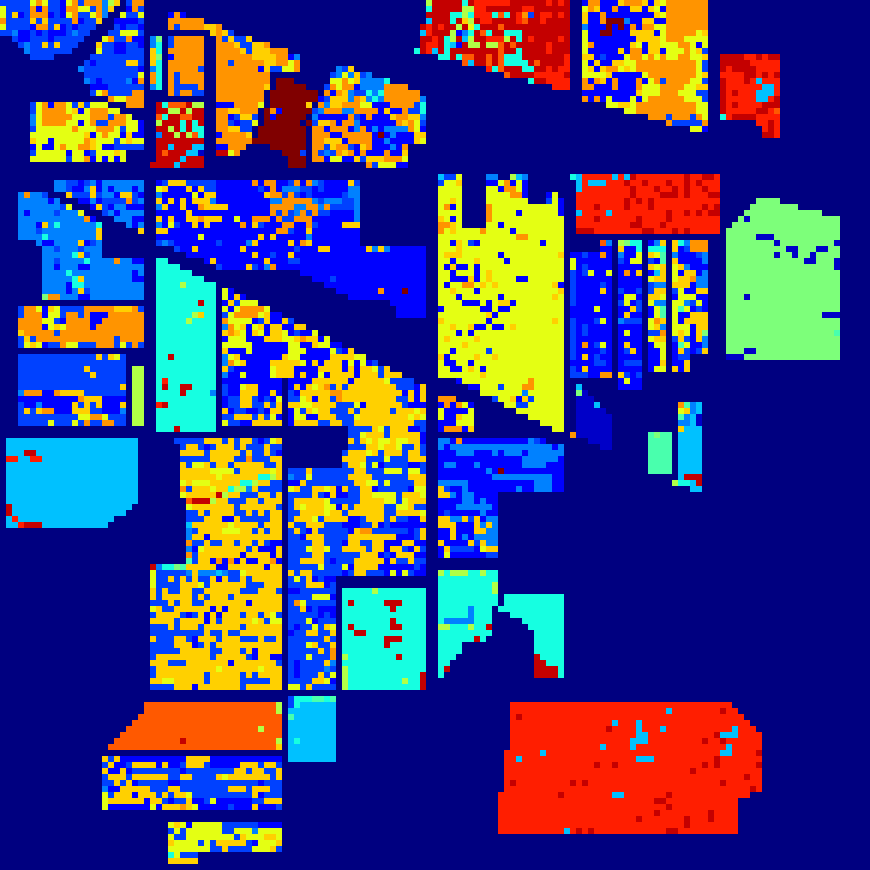}}
\end{minipage}
\caption{Classification outcomes of the proposed $R^2KM$ model and the existing models on the Indian pines dataset.}
\label{Classification results on the Indian pines dataset}
\end{figure*}

\begin{table}[ht!]
\centering
    \caption{Accuracies of the proposed $R^2KM$ against the baseline models over the Indian Pines dataset}
    \label{Performance comparison dataset Indian Pines dataset}
    \resizebox{0.85\linewidth}{!}{
\begin{tabular}{cccccc}
\hline
 & RVFLwoDL \cite{huang2006extreme} & RVFL \cite{pao1994learning} & NF-RVFL \cite{sajid2024neuro} & RKM \cite{suykens2017deep} & $R^2KM$ \\ \hline
1 & $97.78$ & $96.83$ & $98.73$ & $89.86$ & $95.08$ \\
2 & $86.95$ & $90.6$ & $91.64$ & $86.39$ & $94.52$ \\
3 & $59.41$ & $69.95$ & $64.76$ & $57.1$ & $72.97$ \\
4 & $82.61$ & $65.22$ & $56.52$ & $84.62$ & $91.3$ \\
5 & $50$ & $60$ & $50$ & $50$ & $50$ \\
6 & $98.94$ & $98.68$ & $98.94$ & $94.1$ & $99.21$ \\
7 & $66.86$ & $77.52$ & $70.18$ & $67.54$ & $78.9$ \\
8 & $55.46$ & $55.54$ & $48.75$ & $38.15$ & $64.5$ \\
9 & $81.02$ & $85.4$ & $88.32$ & $75.58$ & $86.13$ \\
10 & $57.95$ & $61.92$ & $57.67$ & $47.04$ & $76.85$ \\
11 & $84.18$ & $84.18$ & $89.25$ & $49.04$ & $78.9$ \\
12 & $78.57$ & $78.57$ & $78.57$ & $87.5$ & $92.86$ \\
13 & $85.11$ & $87.23$ & $80.85$ & $93.15$ & $97.87$ \\
14 & $76.22$ & $74.13$ & $73.43$ & $54.37$ & $75.17$ \\
15 & $89.7$ & $89.79$ & $87.47$ & $86.02$ & $85.24$ \\
16 & $100$ & $100$ & $99.05$ & $96.76$ & $98.1$ \\
\hline
oa & $71.38$ & $74.4$ & $70.85$ & $62.61$ & $77.89$ \\ \hline
aa & $78.17$ & $79.72$ & $77.13$ & $72.33$ & $83.6$ \\ \hline
kappa & $67.49$ & $70.95$ & $67.05$ & $58.15$ & $74.8$ \\ \hline
\end{tabular}}
\end{table}

\begin{figure*}[ht!]
\begin{minipage}{.16\linewidth}
\centering
\subfloat[Ground truth]{\includegraphics[scale=0.27]{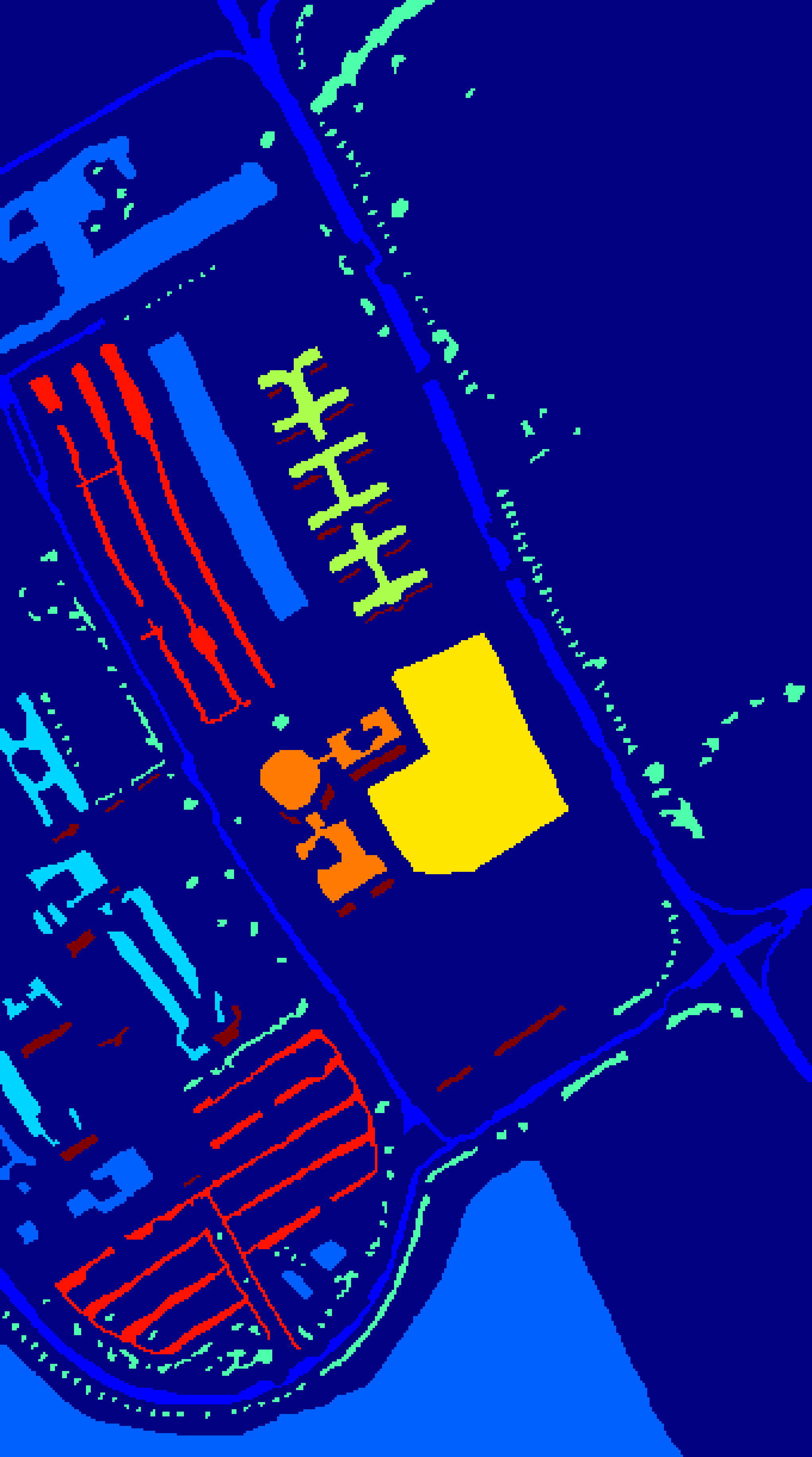}}
\end{minipage}
\begin{minipage}{.16\linewidth}
\centering
\subfloat[RVFL]{\includegraphics[scale=0.27]{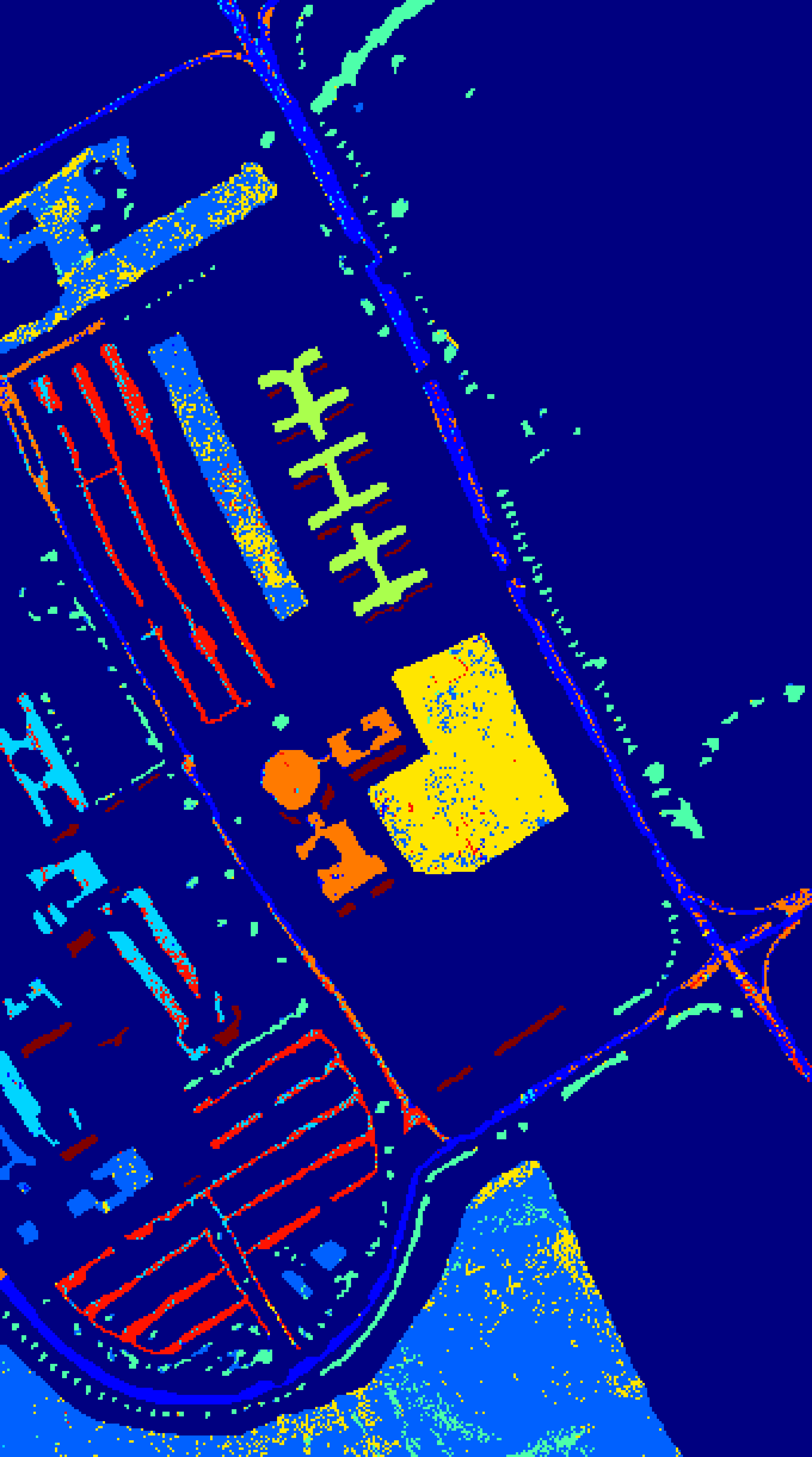}}
\end{minipage}
\begin{minipage}{.16\linewidth}
\centering
\subfloat[RVFLwoDL]{\includegraphics[scale=0.27]{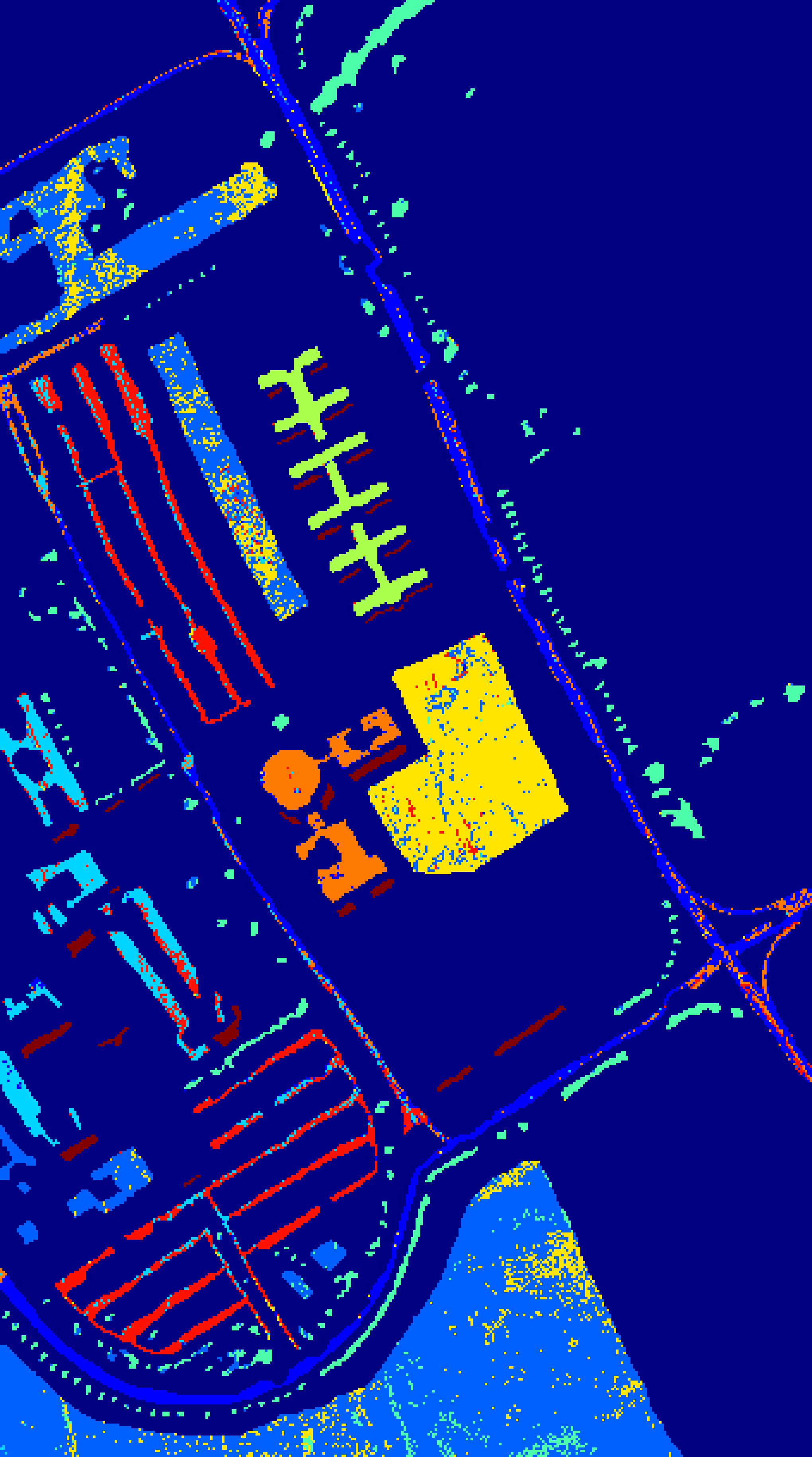}}
\end{minipage}
\begin{minipage}{.16\linewidth}
\centering
\subfloat[NF-RVFL]{\includegraphics[scale=0.27]{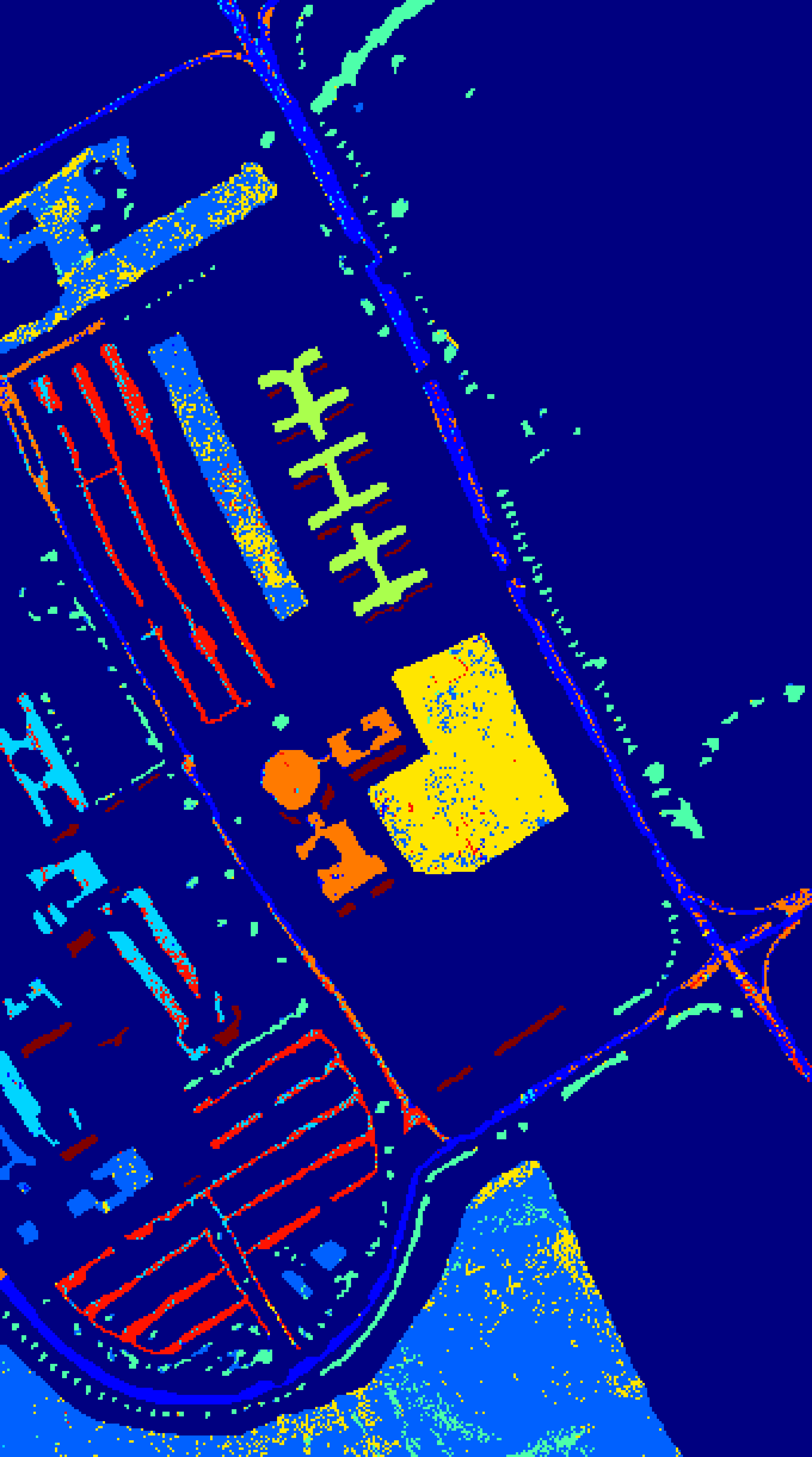}}
\end{minipage}
\begin{minipage}{.16\linewidth}
\centering
\subfloat[RKM]{\includegraphics[scale=0.27]{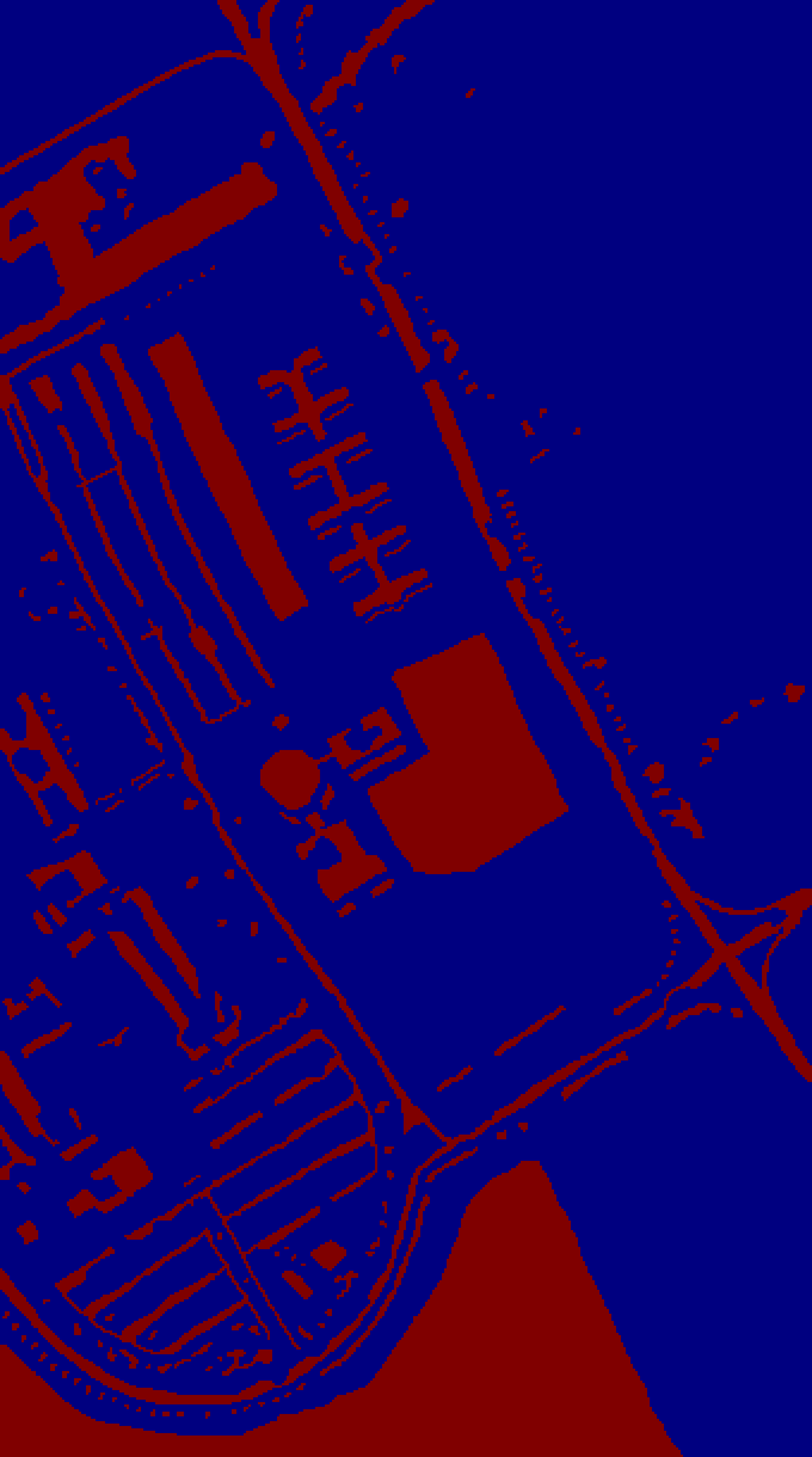}}
\end{minipage}
\begin{minipage}{.16\linewidth}
\centering
\subfloat[$R^2KM$]{\includegraphics[scale=0.27]{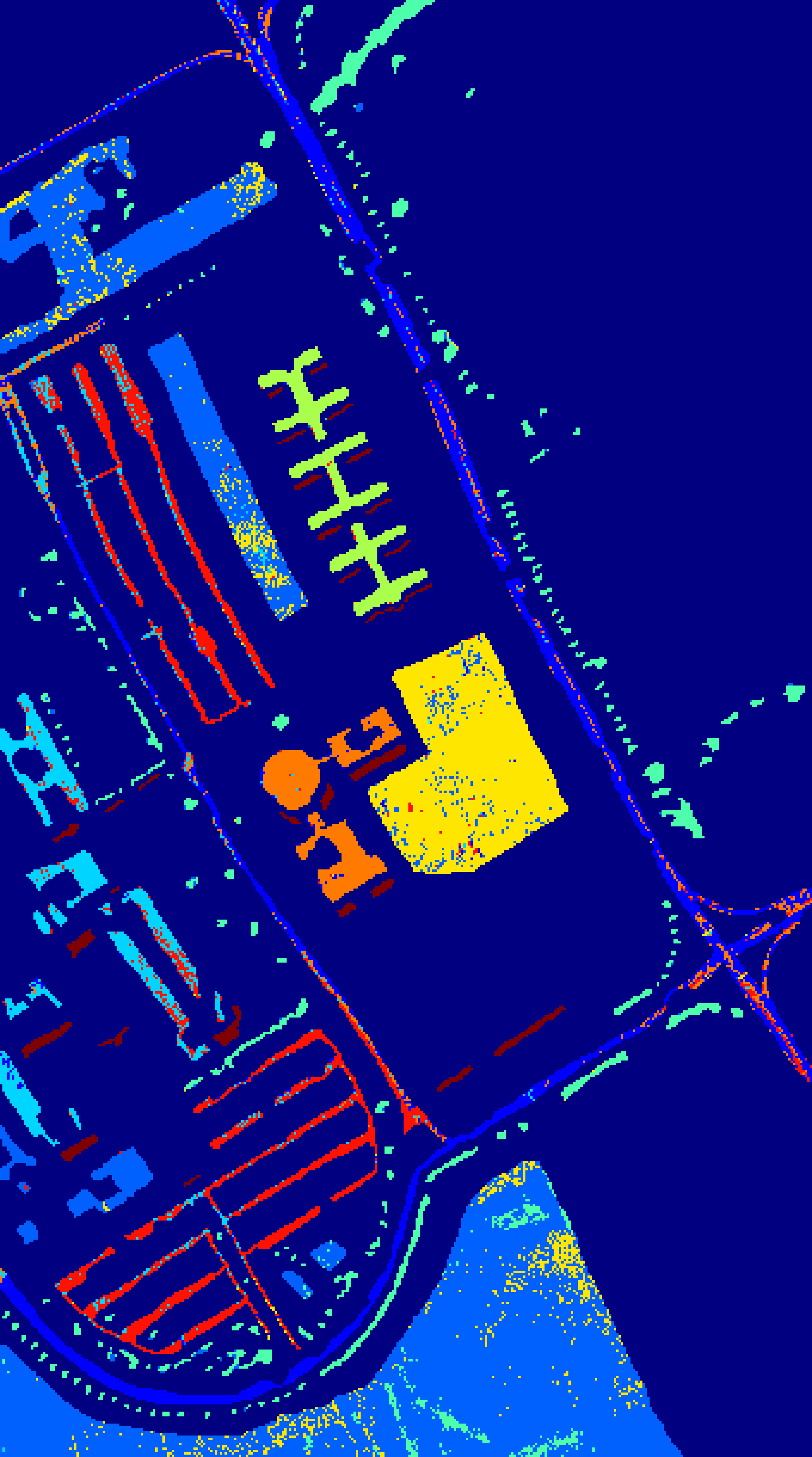}}
\end{minipage}
\caption{Classification outcomes of the proposed $R^2KM$ model and the existing models on the Pavia University dataset.}
\label{Classification results on the Pavia University dataset}
\end{figure*}

To assess the effectiveness of the proposed $R^2KM$ model, we perform experiments using the Indian Pines, KSC, University of Pavia, and Salinas datasets. We compared the performance of our proposed $R^2KM$ model against the baselines RVFLwoDL, RVFL, NF-RVFL, and RKM. Table \ref{Performance comparison dataset Indian Pines dataset} indicates that the proposed model achieves $99.21\%$ accuracy in the `Hay-windrowed' category, $97.87\%$ accuracy in the `Stone-Steel-Towers' category, and $95.08\%$ accuracy in the `Grass-trees' category. The proposed $R^2KM$ model outperforms other models on the three indicators: overall accuracy (OA), average accuracy (AA), and Kappa. Our proposed model outperformed the second-best baseline model, showing an increase in OA, AA, and Kappa values by $3.49$, $3.88$, and $3.85$, respectively. Fig. \ref{Classification results on the Indian pines dataset} illustrates the classification results of the proposed $R^2KM$ model and the baseline models on the Indian Pines dataset, where the $R^2KM$ model shows the fewest misclassified labels, aligning most closely with the ground truth. This indicates that the $R^2KM$ model effectively prioritizes essential spectral information and improves the extraction capabilities of spectral features.

\begin{table}[ht!]
\centering
    \caption{Accuracies of the proposed $R^2KM$ against the baseline models over the University of Pavia dataset}
    \label{Performance comparison dataset Pavia dataset}
    \resizebox{0.85\linewidth}{!}{
\begin{tabular}{cccccc}
\hline
 & RVFLwoDL \cite{huang2006extreme} & RVFL \cite{pao1994learning} & NF-RVFL \cite{sajid2024neuro} & RKM \cite{suykens2017deep} & $R^2KM$ \\ \hline
1 & $76.94$ & $73.63$ & $72.07$ & $68.52$ & $76.67$ \\
2 & $83.84$ & $84.45$ & $87.62$ & $62.99$ & $88.97$ \\
3 & $82.49$ & $83.79$ & $81.59$ & $71.43$ & $85.09$ \\
4 & $91.77$ & $92.78$ & $95.11$ & $94.81$ & $95.14$ \\
5 & $98.96$ & $99.52$ & $98.88$ & $99.09$ & $98.55$ \\
6 & $89.47$ & $89.55$ & $68.63$ & $74.91$ & $92.09$ \\
7 & $93.5$ & $94.47$ & $92.52$ & $89.08$ & $95.45$ \\
8 & $81.99$ & $81.49$ & $73.12$ & $79.82$ & $83.28$ \\
9 & $100$ & $99.88$ & $98.82$ & $99.78$ & $99.88$ \\ \hline
OA & $84.82$ & $84.72$ & $82.54$ & $72.11$ & $87.88$ \\ \hline
AA & $88.77$ & $88.84$ & $85.37$ & $82.27$ & $90.57$ \\ \hline
Kappa & $80.32$ & $80.19$ & $88.33$ & $65.3$ & $84.16$ \\ \hline
\end{tabular}}
\end{table}

\begin{figure*}[htp]
\begin{minipage}{.16\linewidth}
\centering
\subfloat[Ground truth]{\includegraphics[scale=0.41]{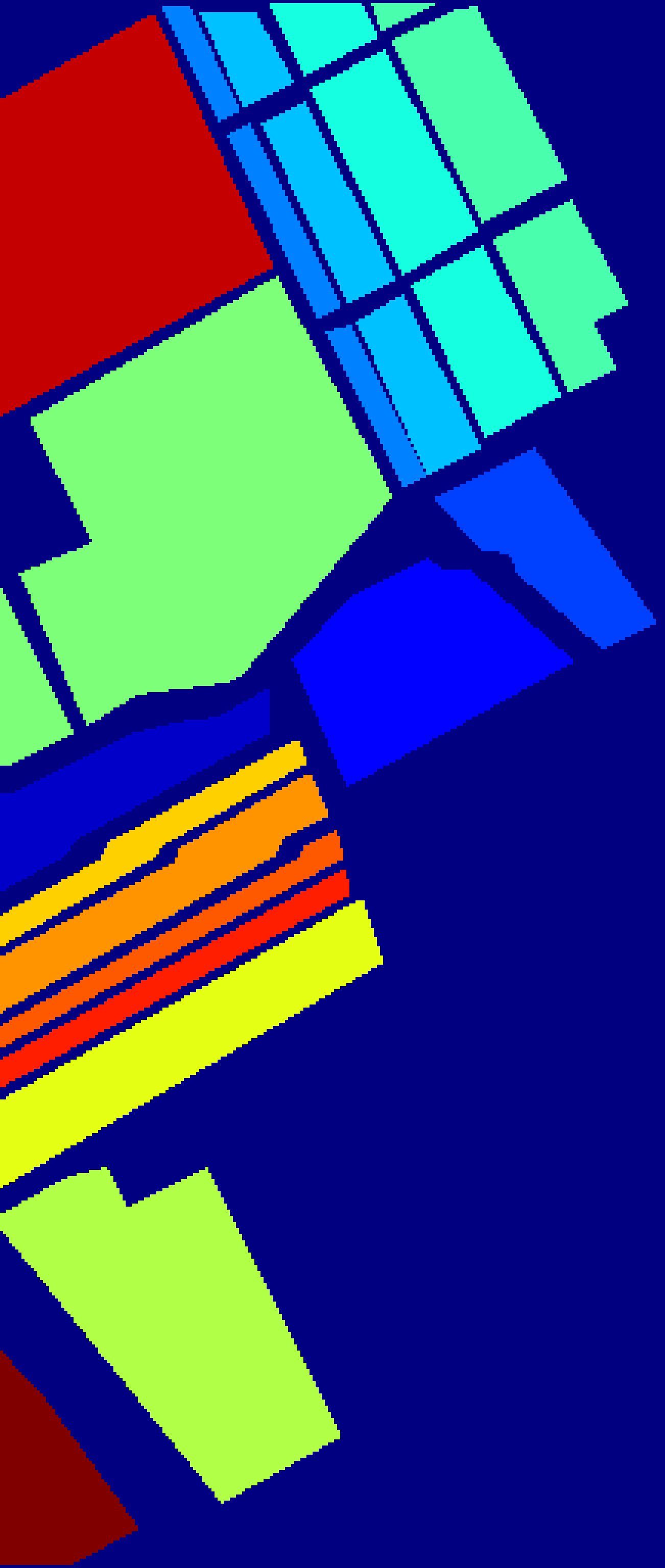}}
\end{minipage}
\begin{minipage}{.16\linewidth}
\centering
\subfloat[RVFL]{\includegraphics[scale=0.41]{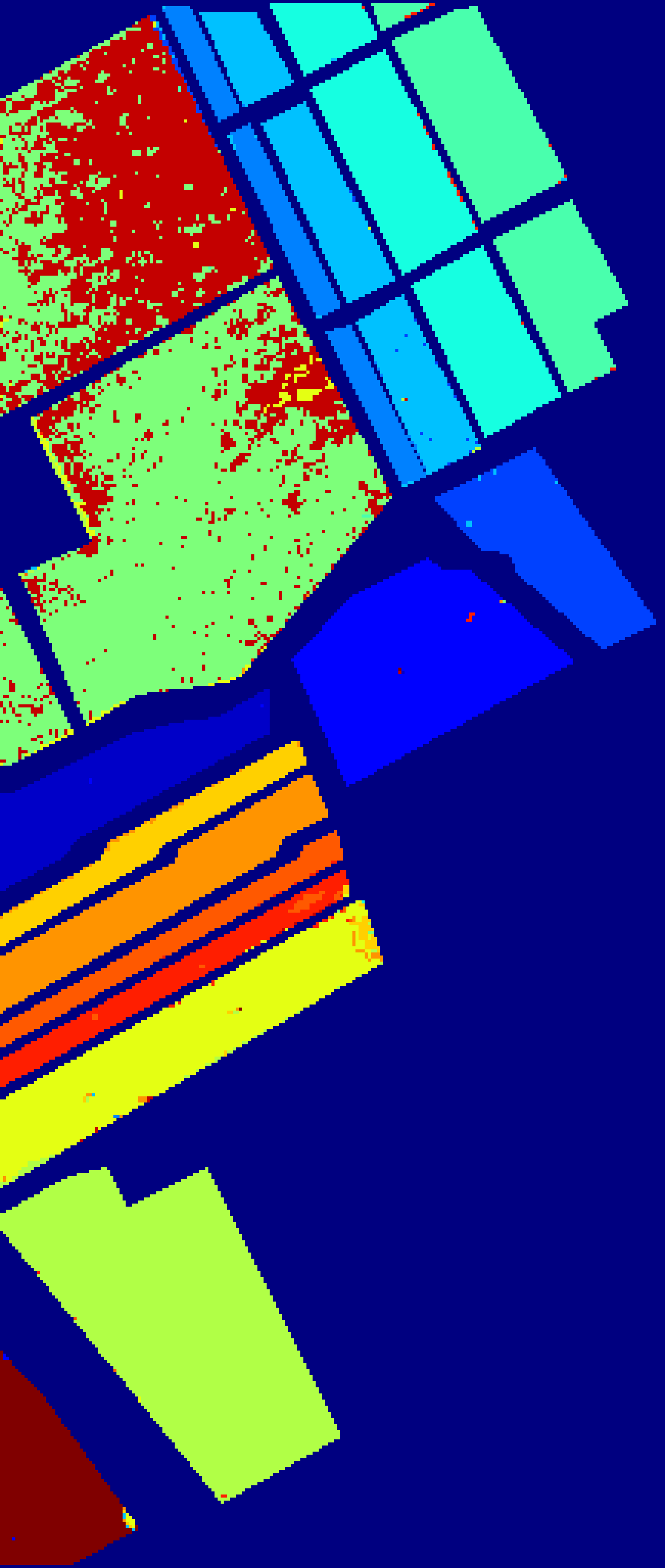}}
\end{minipage}
\begin{minipage}{.16\linewidth}
\centering
\subfloat[RVFLwoDL]{\includegraphics[scale=0.41]{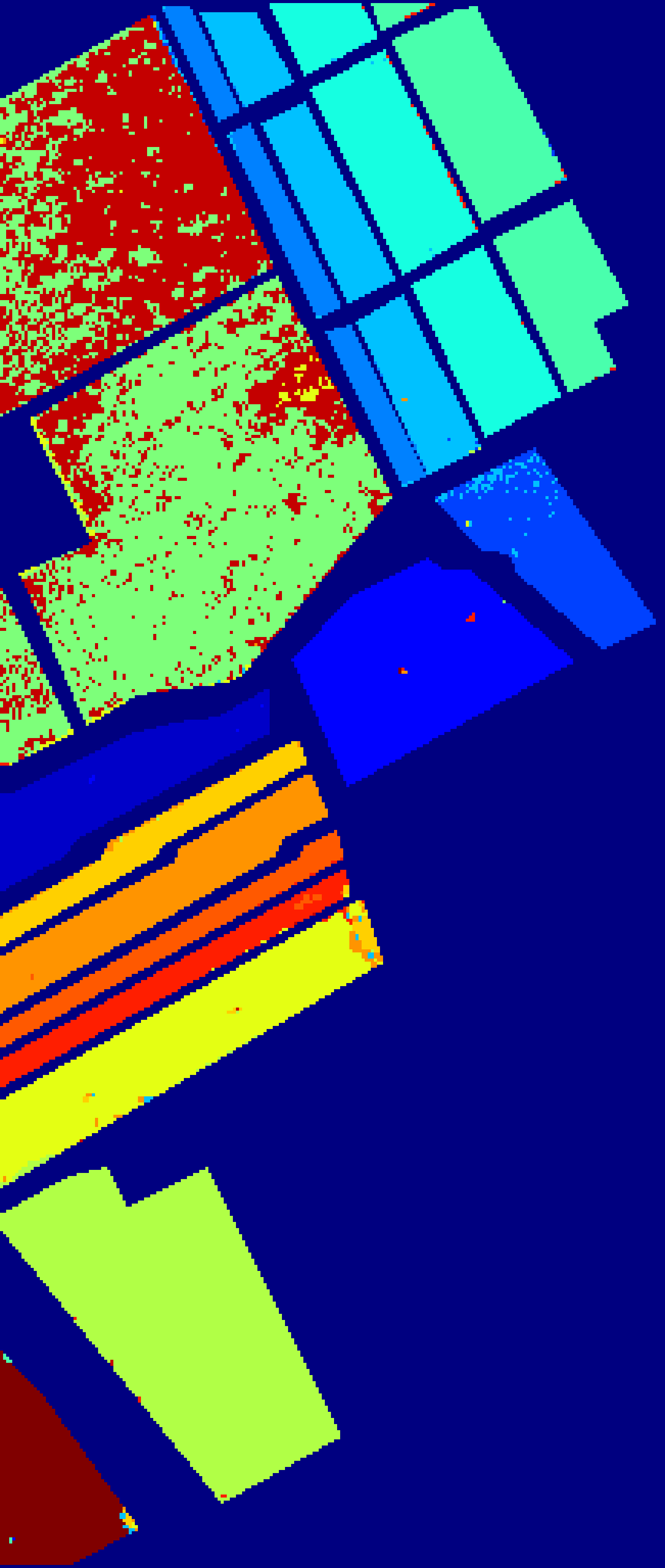}}
\end{minipage}
\begin{minipage}{.16\linewidth}
\centering
\subfloat[NF-RVFL]{\includegraphics[scale=0.41]{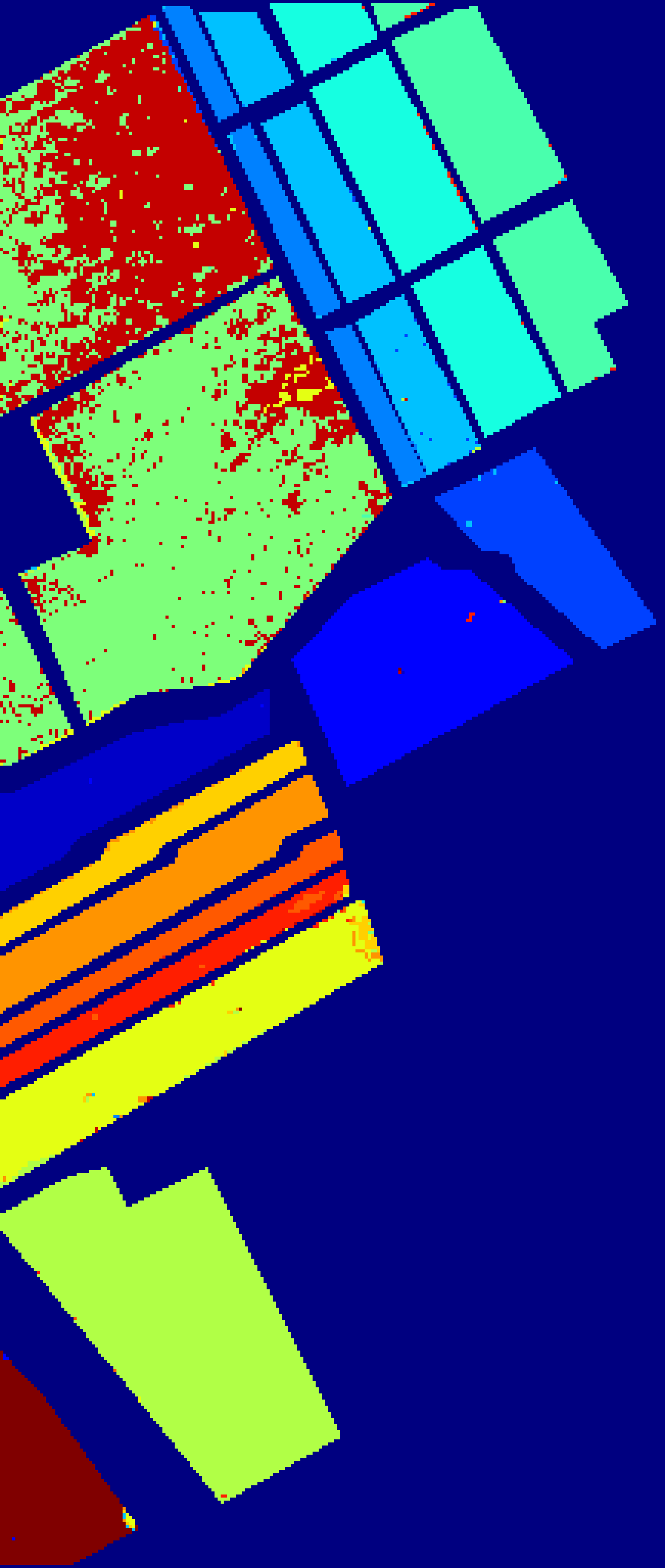}}
\end{minipage}
\begin{minipage}{.16\linewidth}
\centering
\subfloat[RKM]{\includegraphics[scale=0.41]{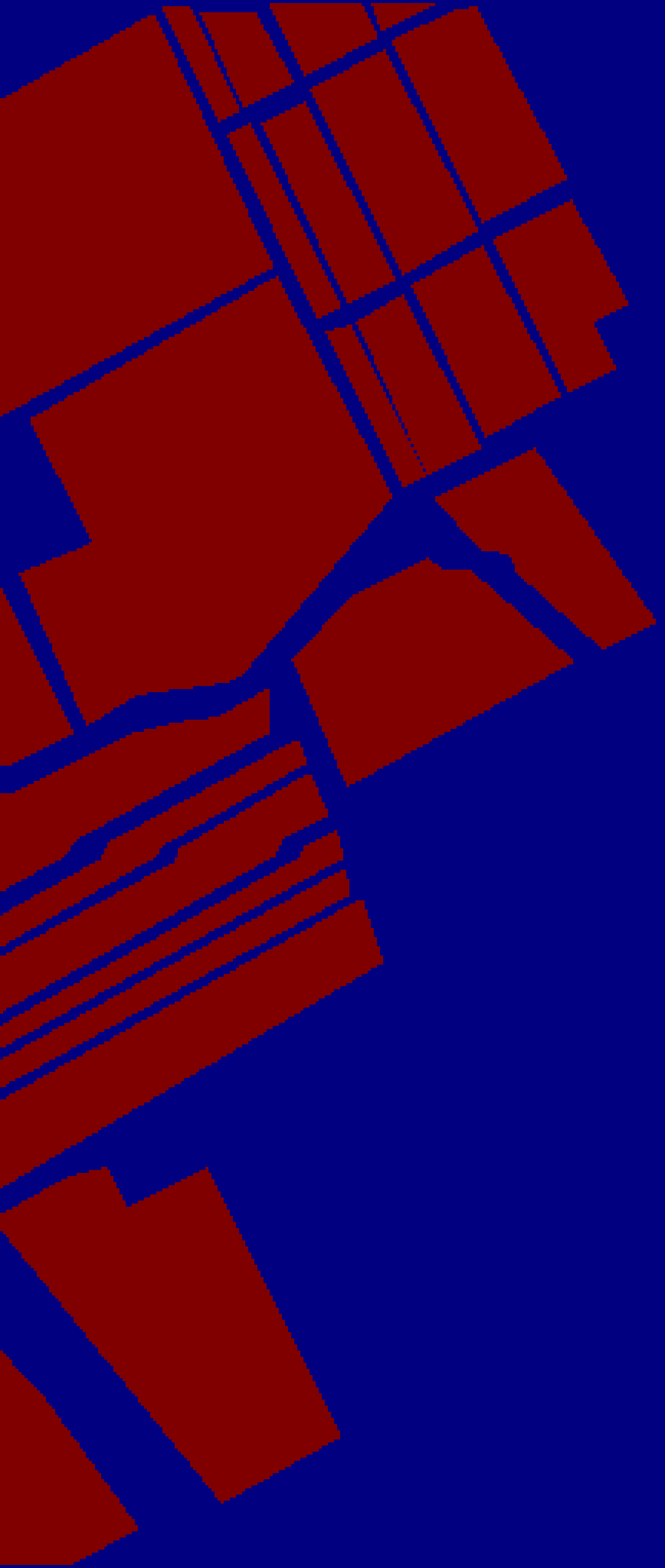}}
\end{minipage}
\begin{minipage}{.16\linewidth}
\centering
\subfloat[$R^2KM$]{\includegraphics[scale=0.41]{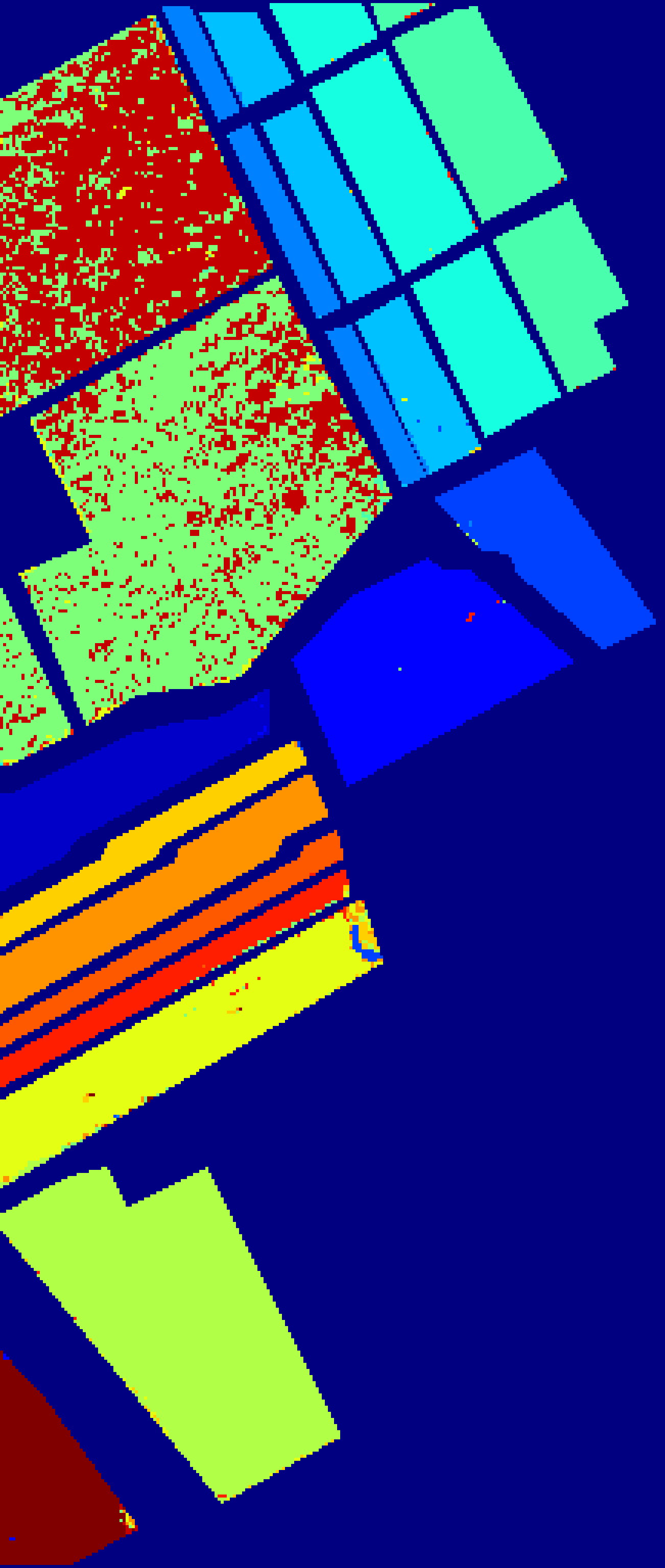}}
\end{minipage}
\caption{Classification outcomes of the proposed $R^2KM$ model and the existing models on the Salinas dataset.}
\label{Classification results on the Salinas dataset}
\end{figure*}

Table \ref{Performance comparison dataset Pavia dataset} shows the classification results of the proposed $R^2KM$ model against the baseline models on the University of Pavia dataset. The proposed $R^2KM$ model attains the highest classification performance, achieving an OA of $87.88\%$, an AA of $90.57\%$, and a Kappa coefficient of $84.16\%$. These results indicate that the $R^2KM$ model achieves the highest accuracy in most categories and reaches a perfect classification accuracy of $99.88\%$. In comparison, the baseline models have difficulty accurately classifying the Bricks and Gravel classes, whereas the proposed model shows a substantial improvement, outperforming the second-best model by $1.29\%$ and $1.30\%$, respectively. Additionally, Fig. \ref{Classification results on the Pavia University dataset} illustrates comparison images that highlight the proposed model's strong classification performance, particularly in the Meadows, Gravel, and Bricks classes, where it demonstrates fewer misclassifications. These results validate the effectiveness of the proposed $R^2KM$ model in extracting spectral features and achieving improved accuracy in feature classification.

\begin{table}[ht!]
\centering
    \caption{Accuracies of the proposed $R^2KM$ against the baseline models over the Salinas dataset}
    \label{Performance comparison dataset Salinas dataset}
    \resizebox{0.85\linewidth}{!}{
\begin{tabular}{cccccc}
\hline
 & RVFLwoDL \cite{huang2006extreme} & RVFL \cite{pao1994learning} & NF-RVFL \cite{sajid2024neuro} & RKM \cite{suykens2017deep} & $R^2KM$ \\ \hline
1 & $99.67$ & $99.72$ & $99.83$ & $99.83$ & $99.72$ \\
2 & $99$ & $98.92$ & $99.23$ & $99.61$ & $99.69$ \\
3 & $99.69$ & $99.84$ & $99.63$ & $97.96$ & $99.48$ \\
4 & $94.51$ & $99.52$ & $93.66$ & $98.61$ & $99.68$ \\
5 & $99.22$ & $99.48$ & $99.09$ & $99.61$ & $99.69$ \\
6 & $99.19$ & $98.72$ & $98.95$ & $97.98$ & $98.22$ \\
7 & $77.87$ & $82.21$ & $77.36$ & $78.82$ & $76.47$ \\
8 & $99.4$ & $99.54$ & $99.4$ & $99.28$ & $99.43$ \\
9 & $99.89$ & $99.89$ & $99.44$ & $98.92$ & $99.33$ \\
10 & $93.39$ & $94.49$ & $92.86$ & $92.29$ & $92.64$ \\
11 & $95.25$ & $96.07$ & $95.35$ & $97.11$ & $99.38$ \\
12 & $99.89$ & $100$ & $99.84$ & $99.89$ & $100$ \\
13 & $99.14$ & $98.9$ & $98.41$ & $99.39$ & $99.51$ \\
14 & $96.49$ & $93.3$ & $94.33$ & $97.63$ & $95.88$ \\
15 & $72.66$ & $70.33$ & $67.35$ & $70.4$ & $77.13$ \\
16 & $98.07$ & $98.89$ & $98.65$ & $98.95$ & $98.3$ \\
\hline
OA & $90.53$ & $91.37$ & $89.54$ & $90.4$ & $91$ \\ \hline
AA & $95.21$ & $95.61$ & $94.59$ & $95.39$ & $95.91$ \\ \hline
Kappa & $89.44$ & $90.37$ & $88.33$ & $89.3$ & $89.97$ \\ \hline
\end{tabular}}
\end{table}

\begin{figure*}[ht!]
\begin{minipage}{.16\linewidth}
\centering
\subfloat[Ground truth]{\includegraphics[scale=0.15]{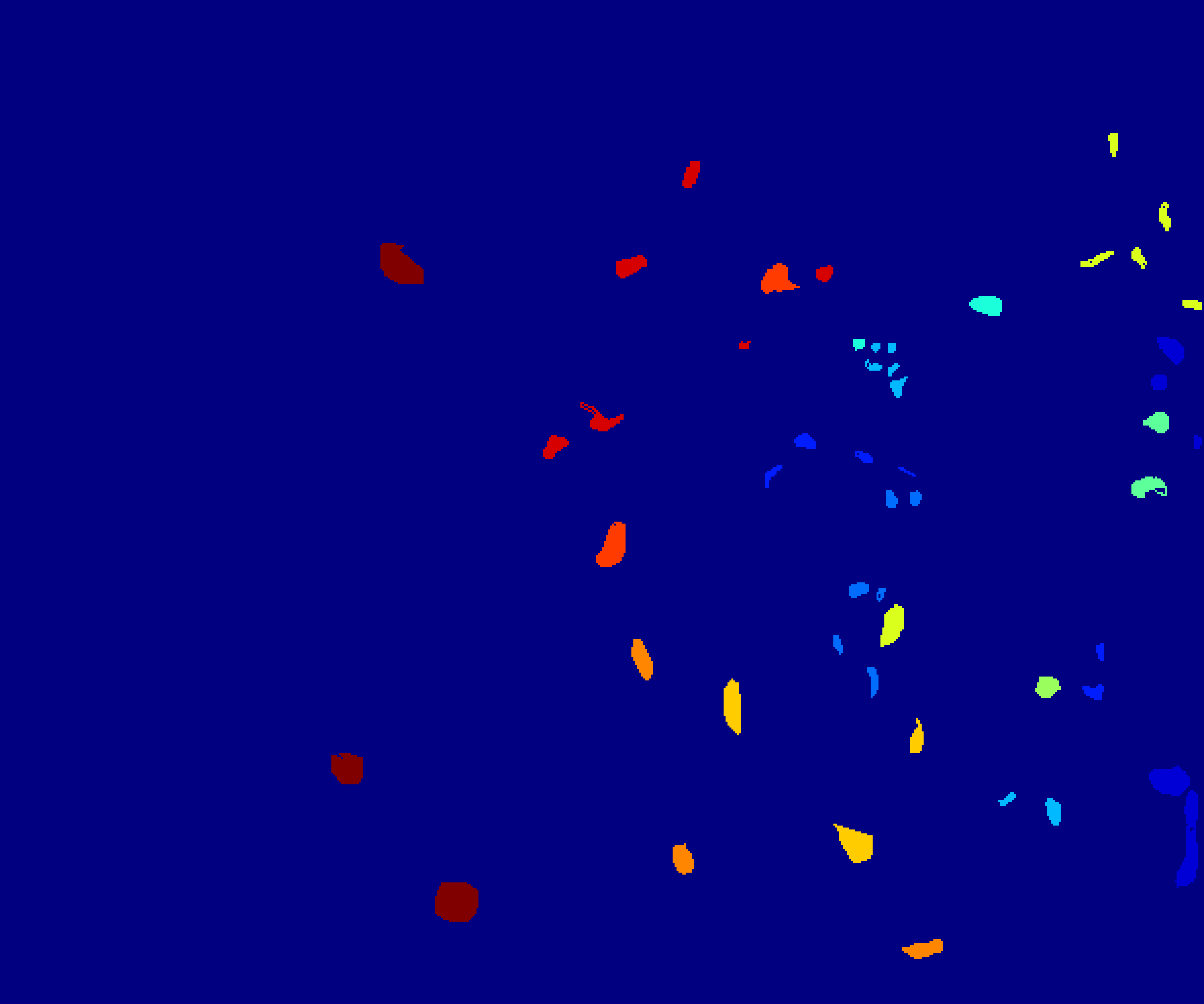}}
\end{minipage}
\begin{minipage}{.16\linewidth}
\centering
\subfloat[RVFL]{\includegraphics[scale=0.15]{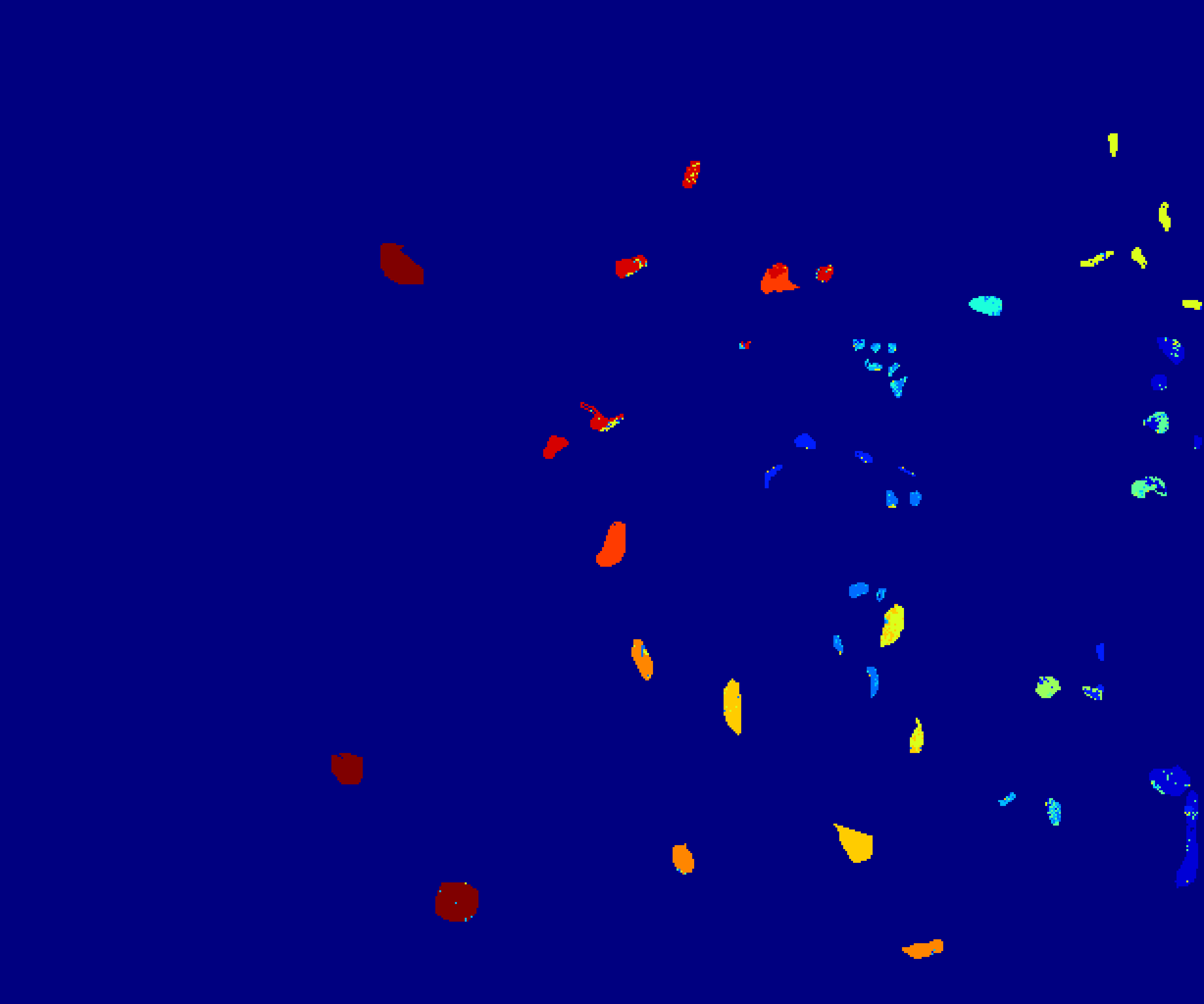}}
\end{minipage}
\begin{minipage}{.16\linewidth}
\centering
\subfloat[RVFLwoDL]{\includegraphics[scale=0.15]{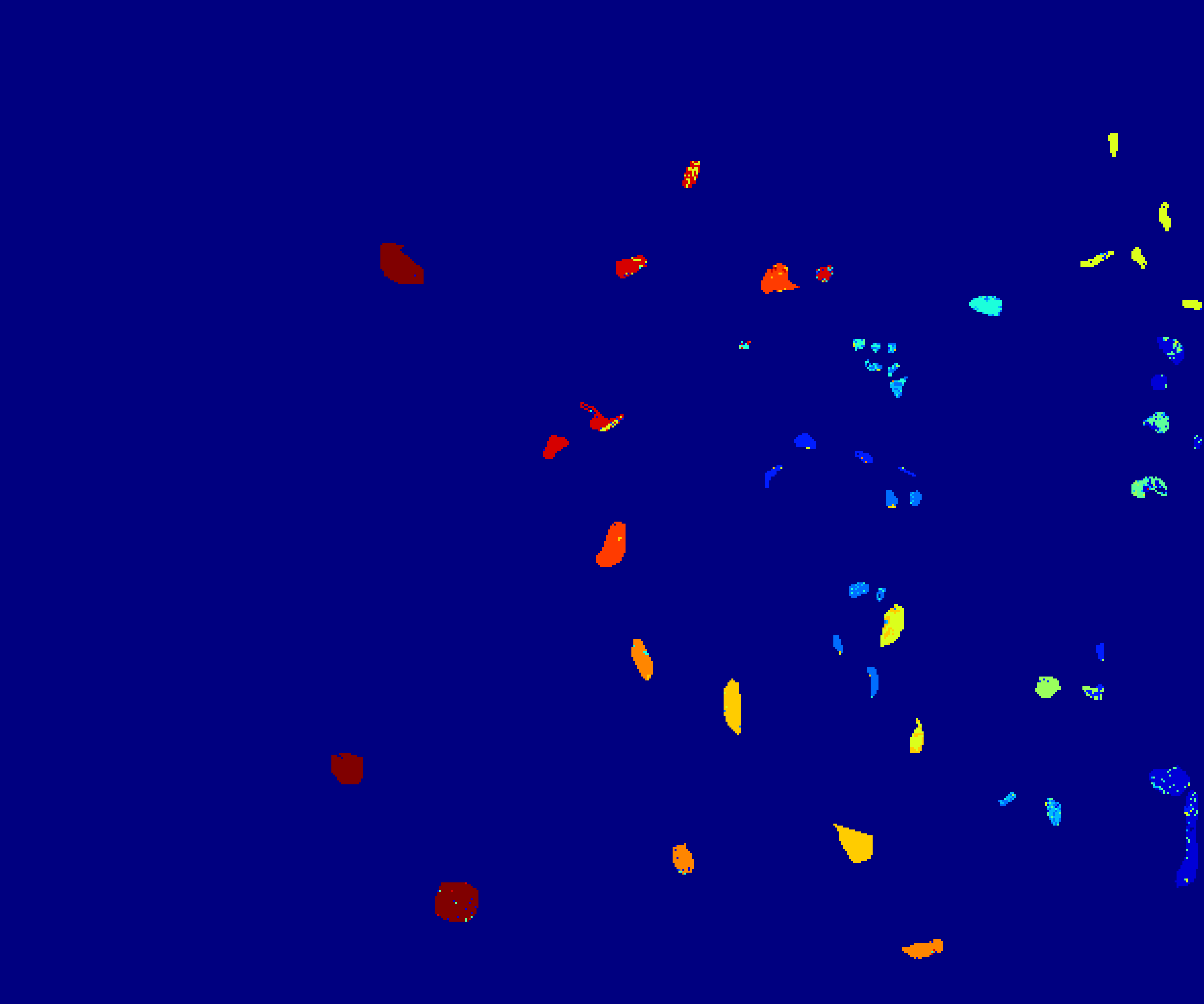}}
\end{minipage}
\begin{minipage}{.16\linewidth}
\centering
\subfloat[NF-RVFL]{\includegraphics[scale=0.15]{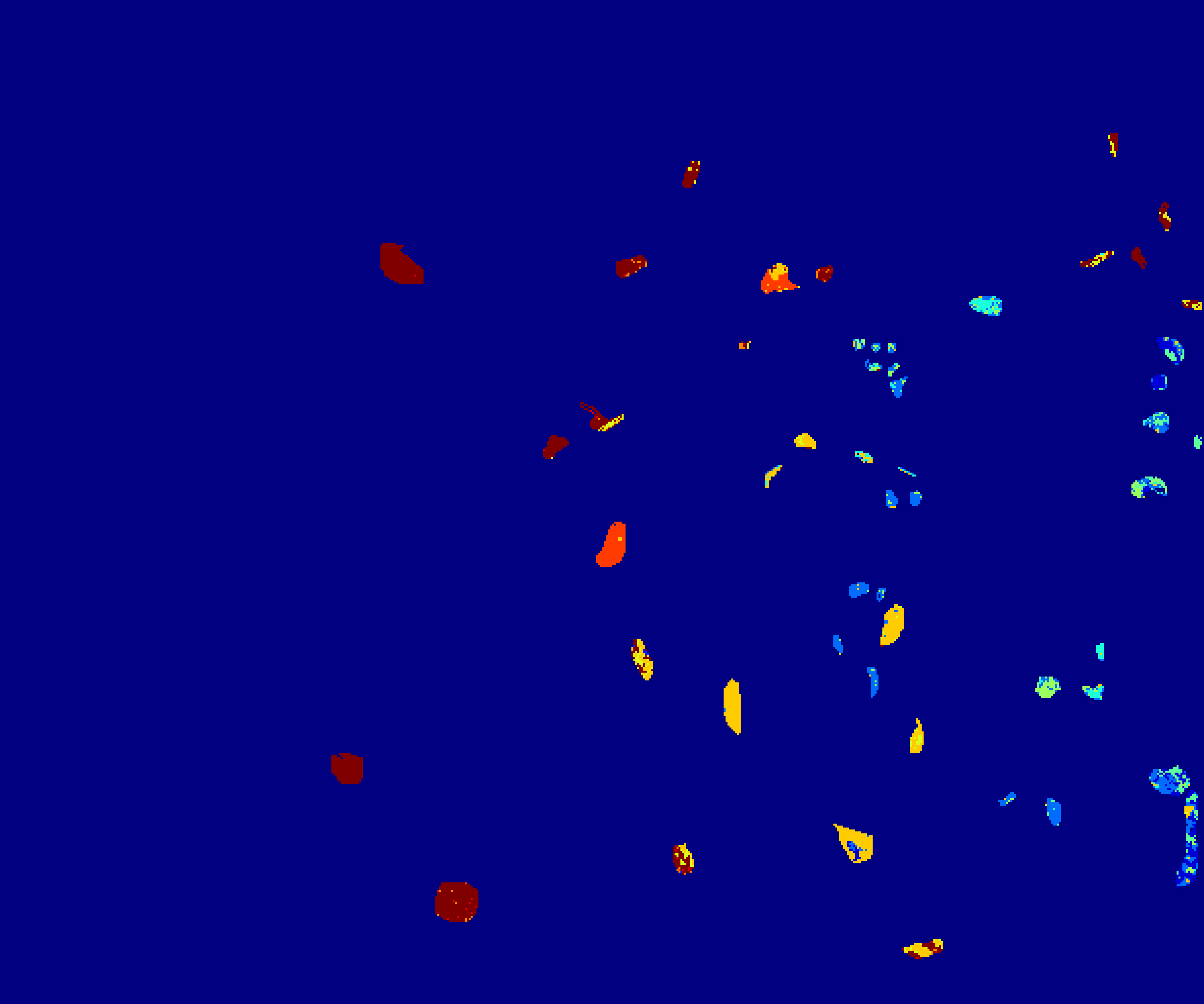}}
\end{minipage}
\begin{minipage}{.16\linewidth}
\centering
\subfloat[RKM]{\includegraphics[scale=0.15]{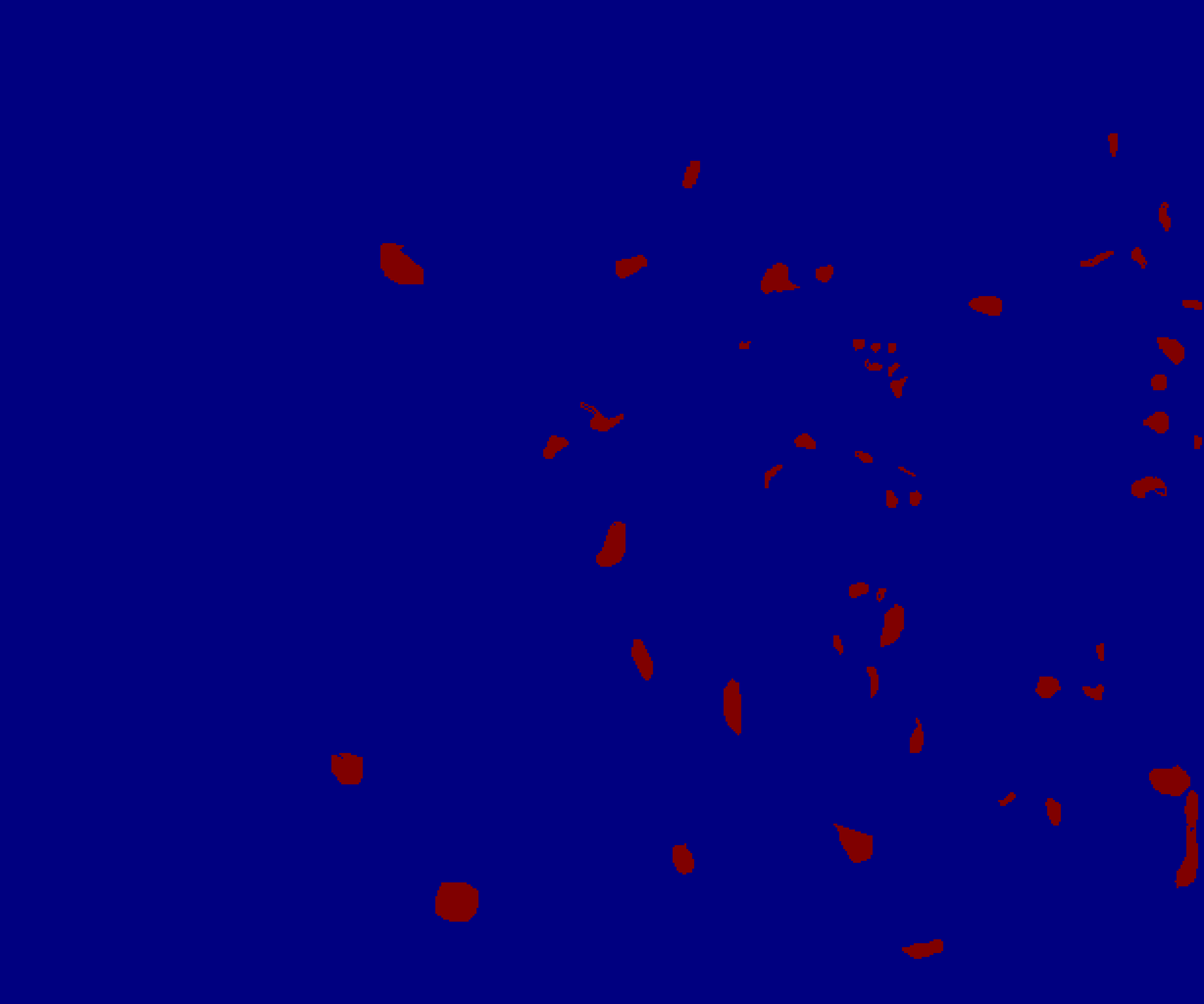}}
\end{minipage}
\begin{minipage}{.16\linewidth}
\centering
\subfloat[$R^2KM$]{\includegraphics[scale=0.15]{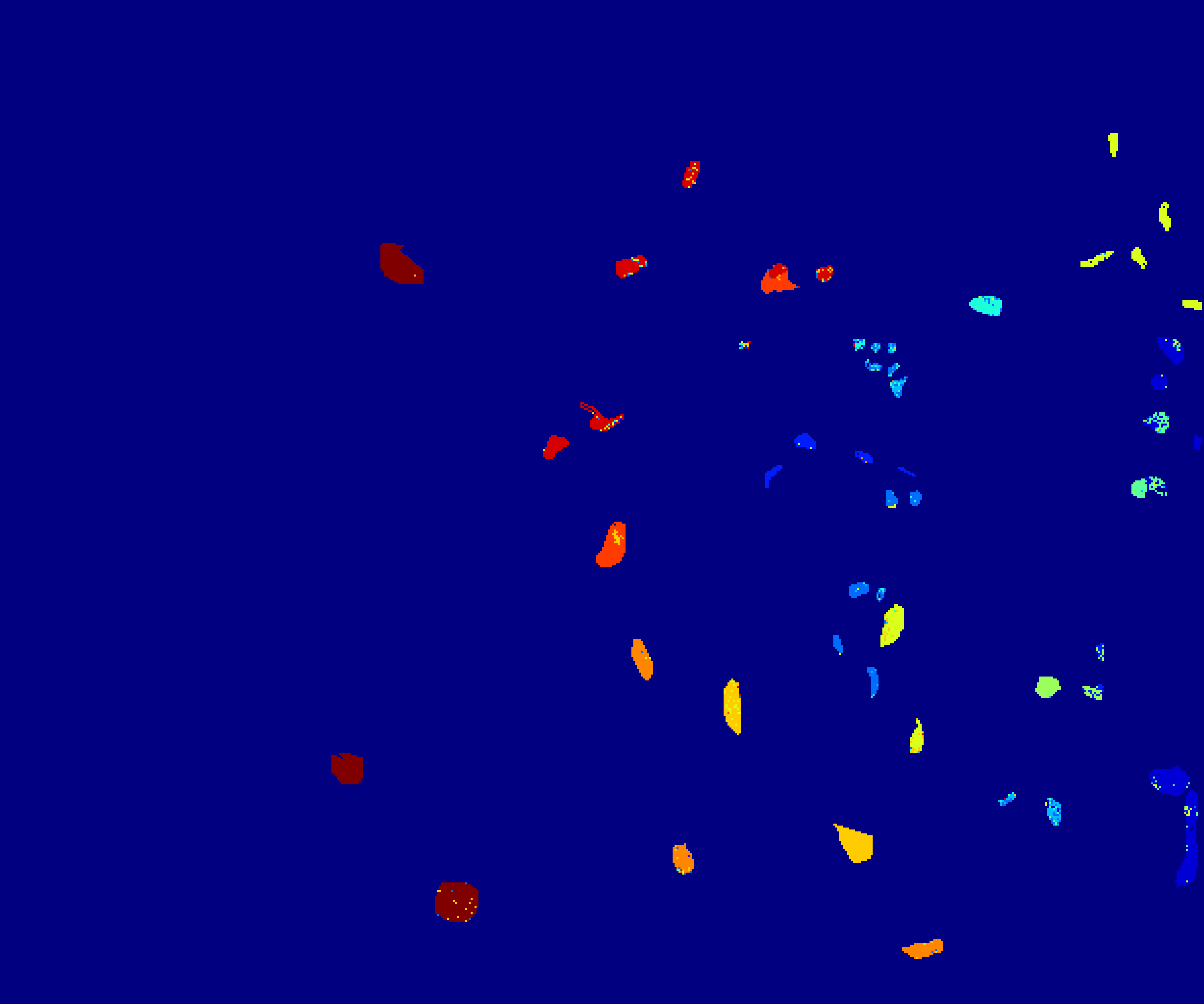}}
\end{minipage}
\caption{Classification outcomes of the proposed $R^2KM$ model and the existing models on the KSC dataset.}
\label{Classification results on the KSC dataset}
\end{figure*}

Table \ref{Performance comparison dataset Salinas dataset} provides a detailed comparison of the experimental accuracy of the $R^2KM$ model with the baseline models on the Salinas dataset. The proposed $R^2KM$ model demonstrates the highest classification performance, with an OA of $91\%$, an AA of $95.91\%$, and a Kappa coefficient of $89.97\%$. Fig. \ref{Classification results on the Salinas dataset} illustrate that, among the compared models, our proposed $R^2KM$ model demonstrates significantly fewer classification errors. Particularly in the Broccoli\_green\_weeds\_2, Broccoli\_green\_weeds\_1, Fallow, Fallow\_rough\_plough and Celery categories, the classification image boundaries are notably clearer, with a significantly reduced number of misclassified points, closely matching the ground truth maps. This demonstrates that the proposed $R^2KM$ model effectively captures the relationships between features in the input images, enhances classifier performance, and aids in the extraction of spatial global features from the imagery.

\begin{table}[ht!]
\centering
    \caption{Accuracies of the proposed $R^2KM$ against the baseline models over the KSC dataset}
    \label{Performance comparison dataset KSC dataset}
    \resizebox{0.85\linewidth}{!}{
\begin{tabular}{cccccc}
\hline
 & RVFLwoDL \cite{huang2006extreme} & RVFL \cite{pao1994learning} & NF-RVFL \cite{sajid2024neuro} & RKM \cite{suykens2017deep} & $R^2KM$ \\ \hline
1 & $97.91$ & $99.34$ & $98.13$ & $97.91$ & $98.13$ \\ 
2 & $78.67$ & $81.37$ & $74.74$ & $72.48$ & $83.44$ \\
3 & $94.49$ & $88.47$ & $95.74$ & $84.21$ & $81.2$ \\
4 & $93.23$ & $93.49$ & $93.49$ & $93.08$ & $93.23$ \\
5 & $88.6$ & $86.4$ & $59.4$ & $88.4$ & $79.8$ \\
6 & $88.32$ & $84.91$ & $89.54$ & $82.41$ & $94.4$ \\
7 & $96.47$ & $84.71$ & $95.29$ & $72.94$ & $98.82$ \\
8 & $58.37$ & $55.98$ & $48.97$ & $51.53$ & $58.37$ \\
9 & $74.47$ & $70.92$ & $73.9$ & $71.23$ & $74.47$ \\
10 & $34.48$ & $34.48$ & $25.86$ & $32.84$ & $36.64$ \\
11 & $87.29$ & $83.05$ & $93.64$ & $90.68$ & $88.56$ \\
12 & $82.96$ & $86.1$ & $83.41$ & $75.45$ & $77.13$ \\
13 & $86.1$ & $90.01$ & $95.41$ & $43.99$ & $93.52$ \\
\hline
OA & $85.38$ & $85.05$ & $81.52$ & $82.55$ & $85.48$ \\ \hline
AA & $81.64$ & $79.94$ & $79.04$ & $73.63$ & $81.36$ \\ \hline
Kappa & $83.7$ & $83.32$ & $79.31$ & $43.17$ & $83.79$ \\ \hline
\end{tabular}}
\end{table}

\begin{table*}[htp]
\centering
    \caption{Accuracies of the proposed $R^2KM$ against the baseline models over UCI and KEEL datasets.}
    \label{Average ACC and average rank for UCI and KEEL datasets}
    \resizebox{0.80\linewidth}{!}{
\begin{tabular}{lccccc}
\hline
Dataset & RVFL \cite{pao1994learning} & RVFLwoDL \cite{huang2006extreme} & NF-RVFL \cite{sajid2024neuro} & RKM \cite{suykens2017deep} & $R^2KM$ \\ 
 $(\#Samples \times \#Feature)$ &    ACC (\%)  &  ACC (\%)  &  ACC (\%) &  ACC (\%)   &   ACC (\%)    \\  
  &   $(\mathcal{C}, N, Act fun)$  &   $(\mathcal{C}, N, Act fun)$  &   $(\mathcal{C}, N, L, Act fun)$  &  $(\gamma, \eta, \sigma)$  &   $(\eta, \lambda, \sigma)$   \\ \hline
bank & 89.17 & 89.31 & 88.95 & 89.46 & 89.54 \\
$(4521 \times 16)$ & $(1, 163, 1)$  & $(10^{3}, 143, 1)$ & $(10^{4}, 43, 35, 8)$  & $(1, 1, 2^{4})$  &  $(10^{-5}, 10^{5}, 2^{2})$ \\
breast\_cancer\_wisc & 97.62 & 98.1 & 97.62 & 97.62 & 98.57 \\
$(286 \times 10)$ & $(10^{-5}, 63, 2)$  & $(10^{-1}, 43, 8)$  & $(10^{2}, 163, 55, 2)$ & $(10^{4}, 10^{3}, 2^{5})$ & $(10^{-5}, 10^{5}, 2^{-5})$  \\
brwisconsin & 97.56 & 97.56 & 98.05 & 97.07 & 97.56 \\
$(683 \times 9)$ & $(10^{-3}, 63, 8)$  & $(10^{4}, 23, 3)$  & $(10^{-1}, 143, 40, 8)$ &  $(10^{1}, 10^{-5}, 2^{4})$ & $(10^{-3}, 10^{5}, 2^{-5})$ \\
checkerboard\_Data & 85.94 & 82.98 & 85.58 & 86.94 & 87.02 \\
$(690 \times 15)$ & $(10^{-3}, 23, 1)$  & $(10^{-3}, 23, 1)$  & $(10^{4}, 203, 15, 5)$  &  $(10^{1}, 10^{3}, 2^{4})$ & $(10^{-5}, 10^{5}, 2^{5})$ \\
chess\_krvkp & 90.41 & 90.2 & 89.65 & 98.27 & 99.48 \\
$(3196 \times 37)$ & $(10^{5}, 83, 2)$  &  $(10^{-3}, 183, 1)$ & $(10^{4}, 83, 15, 6)$ &  $(10^{-1}, 10^{-5}, 2^{5})$ &  $(10^{-5}, 10^{2}, 2^{5})$ \\
conn\_bench\_sonar\_mines\_rocks & 74.6 & 71.43 & 74.6 & 73.65 & 79.37 \\
$(208 \times 60)$ & $(10^{3}, 203, 1)$  & $(10^{3}, 203, 1)$ &  $(10^{-1}, 143, 35, 4)$ &  $(1, 10^{2}, 2^{4})$ & $(10^{-5}, 10^{-5}, 2^{5})$ \\
crossplane150 & 81.11 & 81.11 & 88.89 & 95.56 & 97.78 \\
$(150 \times 2)$ & $(10^{4}, 83, 9)$  & $(10^{4}, 83, 9)$ &  $(1, 3, 50, 6)$ &  $(10^{1}, 10^{2}, 2^{-5})$ & $(10^{-5}, 10^{-4}, 1)$ \\
cylinder\_bands & 72.73 & 72.73 & 73.38 & 75.82 & 79.22 \\
$(512 \times 35)$ & $(10^{-2}, 123, 2)$  &  $(10^{-2}, 123, 2)$ &  $(10^{-1}, 23, 40, 1)$ & $(10^{-5}, 10^{-5}, 2^{3})$ &  $(10^{-5}, 10^{4}, 2^{5})$ \\
ecoli-0-1\_vs\_5 & 98.61 & 95.83 & 97.22 & 97.22 & 97.22 \\
$(240 \times 7)$ &  $(10^{-2}, 23, 4)$  & $(10^{-1}, 143, 4)$  & $(1, 43, 45, 7)$  &  $(10^{-5}, 10^{-5}, 2^{3})$ & $(10^{-2}, 10^{5}, 2^{5})$ \\
ecoli-0-1\_vs\_2-3-5 & 80.59 & 80.59 & 94.59 & 94.59 & 94.59 \\
$(244 \times 8)$ & $(10^{-2}, 143, 1)$  & $(10^{-2}, 143, 1)$  & $(1, 3, 10, 5)$  & $(10^{2}, 10^{2}, 2^{4})$ & $(10^{-5}, 10^{-5}, 1)$ \\
ecoli0137vs26 & 85.74 & 86.81 & 94.68 & 92.81 & 94.68 \\
$(311 \times 7)$ & $(10^{-4}, 83, 9)$  & $(10^{-2}, 63, 9)$ & $(1, 203, 35, 8)$ & $(10^{2}, 10^{-5}, 2^{1})$ & $(10^{-5}, 10^{2}, 2^{5})$ \\
ecoli-0-4-6\_vs\_5 & 96.72 & 96.72 & 100 & 98.36 & 98.36 \\
$(203 \times 6)$ & $(10^{5}, 163, 1)$  & $(10^{5}, 163, 1)$ & $(10^{2}, 63, 5, 8)$ &  $(10^{-5}, 10^{3}, 1)$ & $(10^{-5}, 10^{1}, 2^{1})$ \\
ecoli-0-6-7\_vs\_5 & 91.94 & 91.97 & 95.08 & 93.94 & 95.45 \\
$(222 \times 7)$ & $(10^{2}, 123, 2)$  & $(10^{-1}, 123, 2)$ & $(10^{1}, 143, 25, 7)$ & $(10^{-5}, 10^{-5}, 1)$ & $(10^{-5}, 10^{-5}, 2^{5})$  \\
ecoli-0-1-4-6\_vs\_5 & 98.81 & 98.81 & 96.43 & 98.81 & 98.81 \\
$(280 \times 7)$ &  $(10^{-2}, 103, 4)$   &  $(10^{-3}, 103, 4)$  & $(10^{1}, 163, 30, 4)$  & $(10^{-5}, 10^{1}, 2^{-1})$  &  $(10^{-5}, 10^{4}, 1)$ \\
ecoli-0-3-4-6\_vs\_5 & 98.39 & 96.77 & 100 & 93.55 & 98.39 \\
$(205 \times 7)$ & $(10^{-1}, 163, 1)$  & $(10^{2}, 23, 3)$ & $(10^{1}, 63, 35, 2)$ & $(10^{-4}, 10^{2}, 2^{-1})$ & $(10^{-5}, 10^{-5}, 1)$ \\
ecoli2 & 82.08 & 81.09 & 88.12 & 90.1 & 90.1 \\
$(336 \times 7)$ & $(10^{-2}, 43, 4)$  & $(10^{-2}, 43, 4)$ & $(10^{-2}, 183, 15, 3)$ & $(10^{1}, 10^{-5}, 1)$ & $(10^{-5}, 10^{4}, 2^{1})$ \\
fertility & 90 & 80 & 73.33 & 83.33 & 90 \\
$(100 \times 9)$ & $(10^{-2}, 63, 1)$  &  $(10^{3}, 203, 7)$ & $(10^{5}, 123, 45, 6)$ & $(10^{-4}, 10^{5}, 2^{4})$ & $(10^{-5}, 10^{4}, 2^{5})$ \\
haberman & 76.09 & 76.09 & 77.17 & 76.09 & 76.09 \\
$(306 \times 4)$ & $(10^{-4}, 143, 9)$  & $(10^{1}, 3, 3)$  & $(10^{1}, 123, 50, 5)$  &  $(10^{-1}, 10^{1}, 1)$ & $(10^{-3}, 10^{5}, 1)$ \\
heart\_hungarian & 72.78 & 72.65 & 79.78 & 82.02 & 82.02 \\
$(294 \times 12)$ & $(10^{-5}, 123, 4)$  & $(10^{-2}, 143, 2)$  & $(10^{-3}, 183, 5, 4)$  & $(10^{1}, 10^{-5}, 2^{4})$ & $(10^{-4}, 10^{5}, 2^{1})$ \\
heart-stat & 81.89 & 81.89 & 87.65 & 87.65 & 88.89 \\
$(270 \times 14)$ & $(10^{-1}, 23, 1)$   & $(10^{-2}, 163, 5)$  & $(10^{4}, 3, 20, 8)$  & $(10^{1}, 10^{3}, 2^{4})$  & $(10^{-5}, 10^{-5}, 2^{-1})$ \\
hill\_valley & 68.41 & 67.31 & 66.7 & 70.33 & 76.1 \\
$(1212 \times 100)$ &  $(10^{1}, 183, 9)$ & $(10^{1}, 183, 9)$ & $(10^{-5}, 3, 5, 2)$ & $(10^{-2}, 10^{-5}, 2^{4})$ & $(10^{-5}, 10^{1}, 2^{-5})$ \\
led7digit-0-2-4-5-6-7-8-9\_vs\_1 & 94.74 & 94.74 & 92.48 & 94.74 & 96.24 \\
$(443 \times 8)$ &  $(1, 23, 5)$ & $(10^{-2}, 83, 1)$  & $(10^{-1}, 123, 40, 3)$  & $(10^{1}, 10^{-5}, 2^{2})$  &  $(10^{-5}, 10^{4}, 2^{4})$ \\
monks\_3 & 90.41 & 90.41 & 90.92 & 96.41 & 98.2 \\
$(554 \times 7)$ &  $(10^{-5}, 43, 1)$  & $(10^{-5}, 103, 1)$  & $(10^{2}, 3, 40, 2)$  &  $(10^{5}, 10^{2}, 2^{-5})$  & $(10^{-3}, 10^{5}, 2^{4})$ \\
monks\_2 & 80.66 & 75.14 & 76.09 & 81.77 & 81.77 \\
$(601 \times 7)$ & $(10^{5}, 143, 4)$  & $(10^{5}, 183, 3)$ & $(10^{1}, 203, 5, 9)$ & $(10^{-5}, 10^{3}, 1)$ & $(10^{-5}, 10^{-5}, 1)$ \\
monks\_1 & 81.44 & 82.04 & 94.31 & 94.01 & 95.81 \\
$(556 \times 6)$ & $(10^{5}, 3, 8)$  & $(10^{3}, 3, 6)$ & $(10^{2}, 3, 10, 7)$ & $(10^{5}, 10^{2}, 2^{-5})$ & $(10^{-3}, 10^{5}, 2^{5})$ \\
new-thyroid1 & 86 & 86 & 100 & 98.46 & 98.46 \\
$(215 \times 5)$ & $(10^{1}, 163, 1)$  & $(10^{1}, 163, 1)$ & $(10^{-1}, 123, 15, 2)$ & $(1, 10^{3}, 2^{1})$ & $(10^{-5}, 10^{-5}, 2^{-3})$ \\
oocytes\_trisopterus\_nucleus\_2f & 80.12 & 80.94 & 81.02 & 87.96 & 85.77 \\
$(912 \times 25)$ & $(10^{2}, 63, 1)$  & $(10^{-1}, 203, 9)$ & $(10^{1}, 123, 15, 3)$  & $(10^{-1}, 1, 2^{5})$ & $(10^{-5}, 10^{4}, 2^{5})$ \\
oocytes\_merluccius\_nucleus\_4d & 80.71 & 82.74 & 75.18 & 75.06 & 83.71 \\
$(1022 \times 42)$ & $(10^{1}, 43, 9)$  & $(10^{-1}, 183, 2)$  & $(10^{4}, 23, 5, 7)$   &  $(10^{-1}, 10^{-1}, 2^{5})$ &  $(10^{-5}, 10^{3}, 2^{5})$  \\
parkinsons & 74.75 & 83.22 & 89.83 & 88.14 & 96.61 \\
$(195 \times 22)$ &  $(10^{5}, 203, 7)$ & $(1, 163, 3)$ & $(10^{2}, 203, 5, 3)$ & $(10^{-2}, 10^{-5}, 2^{1})$  & $(10^{-5}, 10^{3}, 2^{2})$ \\
ripley & 86.33 & 86.33 & 90.4 & 89.07 & 89.33 \\
$(1250 \times 2)$ & $(10^{2}, 203, 9)$  & $(10^{2}, 203, 9)$ & $(10^{4}, 103, 30, 6)$ & $(10^{-1}, 10^{-5}, 2^{-3})$ &  $(10^{-5}, 10^{5}, 2^{-5})$ \\
tic\_tac\_toe & 96.65 & 96.65 & 99.65 & 100 & 100 \\
$(958 \times 10)$ & $(10^{2}, 183, 2)$  & $(10^{3}, 183, 9)$  &  $(10^{-1}, 103, 15, 4)$ & $(10^{-5}, 10^{3}, 2^{-2})$  & $(10^{-5}, 10^{-5}, 2^{5})$  \\
vertebral\_column\_2clases & 81.4 & 82.25 & 84.95 & 89.25 & 91.4 \\
$(310 \times 7)$ &  $(10^{2}, 3, 4)$  & $(10^{1}, 23, 9)$  &  $(1, 43, 45, 9)$ & $(1, 10^{-5}, 2^{5})$  &  $(10^{-5}, 10^{4}, 2^{3})$ \\
votes & 90.47 & 90.47 & 96.95 & 97.71 & 97.71 \\
$(435 \times 16)$ & $(10^{-2}, 63, 2)$  & $(10^{-2}, 63, 9)$ & $(1, 83, 30, 2)$  & $(10^{-1}, 10^{-5}, 2^{3})$ & $(10^{-5}, 10^{4}, 2^{5})$ \\
vowel & 99.33 & 98.99 & 95.62 & 100 & 100 \\
$(988 \times 10)$ & $(10^{-1}, 203, 4)$  & $(10^{-1}, 203, 4)$ & $(10^{3}, 83, 30, 5)$ & $(10^{-5}, 10^{3}, 1)$ &  $(10^{-5}, 10^{-5}, 1$ \\
wpbc & 69.49 & 70.97 & 79.66 & 77.97 & 83.05 \\
$(194 \times 34)$ &  $(10^{-1}, 103, 5)$  & $(10^{-1}, 143, 3)$  &  $(10^{1}, 203, 25, 3)$  &  $(10^{1}, 10^{-1}, 2^{4})$  & $(10^{-5}, 10^{4}, 2^{5})$   \\
yeast-2\_vs\_4 & 96.77 & 95.48 & 88.39 & 96.77 & 97.42 \\
$(514 \times 9)$ &  $(10^{2}, 63, 2)$  &  $(10^{2}, 43, 2)$  & $(10^{5}, 163, 15, 5)$    & $(10^{1}, 10^{-1}, 2^{3})$  & $(10^{-5}, 10^{5}, 1)$ \\
yeast-0-2-5-6\_vs\_3-7-8-9 & 89.71 & 89.71 & 94.37 & 94.04 & 94.04 \\
$(1004 \times 9)$ & $(10^{3}, 143, 9)$  &  $(10^{5}, 63, 9)$  &  $(10^{2}, 83, 35, 9)$  & $(10^{2}, 10^{3}, 2^{4})$  & $(10^{-4}, 10^{5}, 2^{2})$  \\
yeast3 & 86.72 & 85.07 & 92.83 & 93.95 & 93.95 \\ 
$(1484 \times 9)$ & $(10^{3}, 123, 7)$  & $(10^{-3}, 83, 1)$  & $(10^{1}, 163, 10, 9)$   &  $(10^{1}, 10^{-1}, 2^{2})$ & $(10^{-5}, 10^{5}, 2^{2})$  \\   \hline
Average ACC & 86.23 & 85.82 & 88.69 & 90.22 & 91.91 \\  \hline
Average rank & 3.8 & 4.09 & 2.96 & 2.58 & 1.57 \\ \hline
\end{tabular}}
\end{table*}

\begin{table*}[ht!]
\centering
    \caption{RMSE, MAE, Poss Error, and Neg Error values on the benchmark UCI datasets for the proposed $R^2KM$ model and the baseline models.}
    \label{Average ACC and average rank for regression datasets}
    \resizebox{0.9\linewidth}{!}{
\begin{tabular}{llccccc} 
\hline
Dataset &  & RVFL \cite{pao1994learning} & RVFLwoDL \cite{huang2006extreme} & NF-RVFL \cite{sajid2024neuro} & RKM \cite{suykens2017deep} & $R^2KM$ \\  \hline
Abalone & RMSE & $0.00009582$ & $0.00006751$ & $0.00037614$ & $0.00044776$ & $0.00005741$ \\
$(4117 \times 7)$ & MAE & $0.00000715$ & $0.00000648$ & $0.00000192$ & $0.00000301$ & $0.00000128$ \\
 & Pos Error & $0.00000893$ & $0.00000627$ & $0.00000221$ & $0.00000309$ & $0.00000136$ \\
 & Neg Error & $0.00000893$ & $0.00000315$ & $0.00000171$ & $0.00000293$ & $0.00000118$ \\ \hline
Airfoil\_Self\_Noise & RMSE & $0.00009255$ & $0.00044539$ & $0.00023369$ & $0.00033556$ & $0.00018123$ \\ 
$(1503 \times 5)$ & MAE & $0.00000057$ & $0.00000363$ & $0.00000682$ & $0.00000518$ & $0.00000014$ \\
 & Pos Error & $0.00000043$ & $0.00000301$ & $0.00000835$ & $0.00000535$ & $0.00000016$ \\
 & Neg Error & $0.00000043$ & $0.00000552$ & $0.00000792$ & $0.00000503$ & $0.00000012$ \\ \hline
California\_Housing & RMSE & $0.00000572$ & $0.00026824$ & $0.00008663$ & $0.00015053$ & $0.00008073$ \\
$(20640 \times 8)$ & MAE & $0.00009267$ & $0.00005778$ & $0.00001526$ & $0.00003775$ & $0.00000985$ \\
 & Pos Error & $0.00000082$ & $0.00000812$ & $0.00000152$ & $0.00000211$ & $0.00004119$ \\
 & Neg Error & $0.00000082$ & $0.00589867$ & $0.00000153$ & $0.00000683$ & $0.00000985$ \\ \hline
Delta\_Ailerons & RMSE & $0.00030406$ & $0.00041134$ & $0.00847544$ & $0.00194797$ & $0.00015698$ \\
$(7129 \times 5)$ & MAE & $0.00000022$ & $0.00001325$ & $0.00000022$ & $0.00001439$ & $0.00000033$ \\
 & Pos Error & $0.00000023$ & $0.00001325$ & $0.00000023$ & $0.00001488$ & $0.00000024$ \\
 & Neg Error & $0.00000023$ & $0.00000953$ & $0.00000021$ & $0.00001393$ & $0.00000037$ \\ \hline
Kinematics\_Robot\_Arm & RMSE & $0.00002883$ & $0.02627194$ & $0.00005232$ & $0.00555962$ & $0.00002045$ \\
$(8192 \times 8)$ & MAE & $0.00000002$ & $0.02170296$ & $0.00000002$ & $0.00005858$ & $0.00000095$ \\
 & Pos Error & $0.00000002$ & $0.02154508$ & $0.00000002$ & $0.00005843$ & $0.00000128$ \\
 & Neg Error & $0.00000002$ & $0.02187257$ & $0.00000002$ & $0.00005895$ & $0.00000073$ \\ \hline
Parkinsons\_Telemonitoring & RMSE & $0.00009335$ & $0.00126815$ & $0.00000028$ & $0.00011901$ & $0.00000027$ \\
$(5875 \times 21)$ & MAE & $0.00004332$ & $0.00216741$ & $0.00002786$ & $0.00018996$ & $0.00000397$ \\
 & Pos Error & $0.00006573$ & $0.00337333$ & $0.00003306$ & $0.00011981$ & $0.00000281$ \\
 & Neg Error & $0.00006573$ & $0.00101106$ & $0.00001798$ & $0.00024732$ & $0.00000423$ \\ \hline
Pole\_Telecomm & RMSE & $0.00000048$ & $0.00400381$ & $0.00014462$ & $0.00013515$ & $0.00000046$ \\
$(15000 \times 48)$ & MAE & $0.00004867$ & $0.00103312$ & $0.00003865$ & $0.00027685$ & $0.00000382$ \\
 & Pos Error & $0.00000048$ & $0.00102632$ & $0.00000044$ & $0.00028667$ & $0.00000418$ \\
 & Neg Error & $0.00000048$ & $0.00104096$ & $0.00000025$ & $0.00026795$ & $0.00000379$ \\ \hline
Triazines & RMSE & $0.00479105$ & $0.83598033$ & $0.00000253$ & $0.00000449$ & $0.00000265$ \\
$(186 \times 60)$ & MAE & $0.00072627$ & $0.58658267$ & $0.00037521$ & $0.34107194$ & $0.00957156$ \\
 & Pos Error & $0.00059803$ & $0.50795267$ & $0.00052313$ & $0.2528012$ & $0.00757198$ \\
 & Neg Error & $0.00006934$ & $0.64555504$ & $0.00020453$ & $0.41225801$ & $0.01157115$ \\ \hline
Yacht\_Hydrodynamics & RMSE & $0.00001585$ & $0.00040787$ & $0.00008135$ & $0.00037425$ & $0.00001718$ \\
$(308 \times 6)$ & MAE & $0.00000128$ & $0.00032524$ & $0.00000577$ & $0.00029559$ & $0.00001697$ \\
 & Pos Error & $0.00008546$ & $0.00044093$ & $0.00000656$ & $0.00029196$ & $0.00007589$ \\
 & Neg Error & $0.00000876$ & $0.00020703$ & $0.00000462$ & $0.00029899$ & $0.00001697$ \\ \hline
Average RMSE &  & $0.00060308$ & $0.0965694$ & $0.00105033$ & $0.00100826$ & $0.00005748$ \\ \hline
Average Rank &  & $2.11$ & $4.44$ & $3.11$ & $3.89$ & $1.44$ \\  \hline
\end{tabular}}
\end{table*}

Table \ref{Performance comparison dataset KSC dataset} illustrates that the performance of our proposed $R^2KM$ model substantially exceeds that of the baseline models. The $R^2KM$ model demonstrates superior classification accuracy, achieving an OA of $85.48\%$, an AA of $81.36\%$, and Kappa coefficient of $83.79\%$, which is notably higher than the best-performing baseline models. This enhanced performance is reflected in its robust ability to correctly classify various ground object categories, including the challenging ones. Our model's capability to effectively utilize spectral features and capture complex patterns in the hyperspectral data is evident from its improved accuracy metrics compared to the baseline models. This confirms the $R^2KM$ model's effectiveness in handling the diverse and intricate nature of hyperspectral data in the KSC dataset. Fig. \ref{Classification results on the KSC dataset} highlights the reduced number of misclassifications in the KSC dataset achieved by our proposed $R^2KM$ model. The $R^2KM$ model exhibits higher classification accuracy than the baseline models, showing significantly fewer errors and a closer alignment with the ground truth. This clearly indicates the effectiveness of our model in minimizing misclassifications and enhancing overall performance.

The experimental results from various datasets, including Indian Pines, University of Pavia, Salinas, and KSC, underscore the enhanced performance of the proposed $R^2KM$ model when compared to the baseline models. Our model consistently outperforms existing models in terms of classification accuracy, as evidenced by higher overall accuracies and fewer misclassifications. Specifically, on the Indian Pines and University of Pavia datasets, our model achieved remarkable improvements in accuracy and reduced misclassifications, showcasing its robustness in various scenarios. Likewise, the proposed model demonstrated improved classification accuracy and a significant reduction in errors for both the Salinas and KSC datasets, outperforming existing models. These results collectively affirm the effectiveness and superiority of the $R^2KM$ model across various hyperspectral image datasets.

\subsection{Results and Discussions on UCI and KEEL Datasets for Classification}

In this section, we conduct an in-depth comparison between the proposed $R^2KM$ model and the baseline models. The comparison is carried out across $38$ benchmark datasets from UCI \cite{dua2017uci} and KEEL \cite{derrac2015keel} repositories. Table \ref{Average ACC and average rank for UCI and KEEL datasets} presents the average accuracy (ACC) of the proposed models compared to the existing models. Table \ref{Average ACC and average rank for UCI and KEEL datasets} displays the optimal hyperparameters for the proposed $R^2KM$ model and the existing models. The average ACC of $R^2KM$ model, and the existing models RVFL, RVFLwoDL, NF-RVFL, and RKM, are $91.91\%$, $86.23\%$, $85.82\%$, $88.69\%$, and $90.22\%$, respectively. The proposed $R^2KM$ model secured the highest average ACC. This indicates that the proposed $R^2KM$ model demonstrates a strong level of confidence in its predictive performance. To overcome the challenges associated with relying solely on average ACC, we implemented a series of statistical tests following the recommendations of \citet{demvsar2006statistical}. These tests are specifically designed for assessing classifier performance across various datasets, especially in situations where the assumptions required for parametric tests are not fulfilled. Through the incorporation of statistical tests, we aim to conduct a thorough assessment of the models' performance, allowing us to formulate broad and impartial conclusions about their effectiveness. In the ranking system, every model is assigned a rank based on its performance across different datasets, which enables a detailed assessment of its overall effectiveness. Models that perform poorly are assigned higher ranks, while those with better performance receive lower ranks. This approach takes into account the compensatory effect, whereby strong performance on certain datasets can offset weaker performance on others. To assess \(k\) models across \(N\) datasets, the rank of the \(q^{th}\) model on the \(p^{th}\) dataset is denoted as \(\mathscr{R}_q^p\). The average rank for the \(q^{th}\) model is calculated as follows: $\mathscr{R}_q = \frac{1}{N}\sum_{p=1}^N\mathscr{R}_q^p$. The average rank of the proposed $R^2KM$ model along with the existing RVFL, RVFLwoDL, NF-RVFL, and RKM models are $1.57$, $3.80$, $4.09$, $2.96$, and $2.58$, respectively. The proposed $R^2KM$ model attained the lowest average rank among the models evaluated. Since a lower rank indicates superior performance, the proposed $R^2KM$ model is identified as the top-performing model. By analyzing the average rankings of the models, the Friedman test \cite{friedman1937use} identifies any significant differences among them. The Friedman test is used to evaluate and compare the performance of the models across various datasets. According to the null hypothesis, all models are assumed to have the same average rank, indicating that their performance levels are comparable. The Friedman test utilizes the chi-squared distribution, denoted as \(\chi_F^2\), which has \((k-1)\) degrees of freedom. The test involves calculating: $\chi_F^2 = \frac{12N}{k(k+1)}\left[ \sum_q \mathscr{R}_q^2 - \frac{k(k+1)^2}{4}  \right]$. The $F_F$ statistic is determined using the formula: $F_F = \frac{(N-1)\chi^2_F}{N(k-1) - \chi^2_F}$, where the $F$-distribution has degrees of freedom $(k - 1)$ and \( (N - 1) \times (k - 1) \). For $k=5$ and $N=38$, we get $\chi_F^2 = 61.5752$ and $F_F = 25.1953$. From the $F$-distribution table at a $5\%$ significance level, the critical value is $F(4, 148) = 2.4328$. Since \( F_F > 2.4328 \), we conclude that the null hypothesis can be rejected, indicating the presence of significant differences between the models. Consequently, we utilize the Nemenyi post hoc test \cite{demvsar2006statistical} to further investigate the pairwise differences among the various models. The critical difference (C.D.) is determined by \( \text{C.D.} = q_\alpha \times \sqrt{{k(k+1)}/{6N}} \), where \( q_\alpha \) denotes the critical value derived from the distribution table specific to the two-tailed Nemenyi test. From the \(F\)-distribution table, the value of C.D. is calculated to be \(0.9895\), where \(q_\alpha = 2.728\) at a significance level of \(5\%\). The variations in average ranks between the proposed $R^2KM$ model and the baseline models, including RVFL, RVFLwoDL, NF-RVFL, and RKM, are $2.23$, $2.52$, $1.39$, and $1.01$, respectively. The results of the Nemenyi post hoc test indicate that the proposed $R^2KM$ model demonstrates statistically significant superiority over the baseline models. We determine that the proposed $R^2KM$ model outperforms the current models in terms of both ranking and overall effectiveness. 

\subsection{Results and Discussions on UCI Datasets for Regression}
Here, we discuss the performance of the proposed $R^2KM$ model by using $9$ benchmark regression datasets from the UCI repository \cite{dua2017uci}. The results are compared against the baseline models: RVFL, RVFLwoDL, NF-RVFL, and RKM. Table \ref{Average ACC and average rank for regression datasets} shows the experimental results of the proposed $R^2KM$ model and the existing models on UCI datasets. The evaluation is based on metrics such as RMSE, MAE, Pos Error, and Neg Error. The evaluation primarily emphasizes RMSE, a crucial metric where lower values indicate superior model performance. The proposed $R^2KM$ model performs well by achieving the lowest RMSE values in $5$ of the $9$ datasets. For the other $3$ datasets, the $R^2KM$ model secures the second-lowest RMSE values. The $R^2KM$ model's consistent performance across diverse datasets underscores its effectiveness. The average RMSE values further validate the effectiveness of the $R^2KM$ model. The existing models have the following average RMSE values: RVFL with $0.00060308$, RVFLwoDL with $0.0965694$, NF-RVFL with $0.00105033$, and RKM with $0.00100826$. In contrast, the proposed $R^2KM$ model achieves a significantly better average RMSE of $0.00005748$, surpassing the performance of the baseline models. To provide a more accurate assessment of model performance, each model should be ranked separately for each dataset rather than relying solely on average RMSE values. Table \ref{Average ACC and average rank for regression datasets} illustrates the average rankings of the proposed $R^2KM$ model and the baseline models. Models are ranked according to RMSE, where the model with the lowest RMSE is assigned the highest rank. The average ranks for the proposed $R^2KM$ model and the existing RVFL, RVFLwoDL, NF-RVFL, and RKM models are $1.44$, $2.11$, $4.11$, $3.11$, and $3.89$, respectively. The rankings demonstrate that the $R^2KM$ model outperforms the baseline models, underscoring its superior performance. To further assess the effectiveness of the proposed $R^2KM$ model, we conducted the Friedman test and subsequently performed the Nemenyi post hoc test. The Friedman test is used to statistically determine the significance of performance differences among the models. For \( p = 5 \) and \( N = 9 \), the obtained values were \( \chi^2_F = 21.7566 \) and \( F_F = 12.2199 \). The \( F_F \) statistic follows an $F$-distribution with degrees of freedom $(4, 32)$. The critical value for \( F_F \) with these degrees of freedom at a $5\%$ significance level is $2.6684$, from $F$-distribution table. Since the calculated \( F_F \) value surpasses $2.6684$, we reject the null hypothesis, indicating that significant differences exist among the models. Subsequently, the Nemenyi post-hoc test is performed to identify significant differences in the pairwise comparisons between the models. The value of C.D. is calculated as $2.0333$. This means that for the average rankings shown in Table \ref{Average ACC and average rank for regression datasets} to be considered statistically significant, there must be at least a $2.0333$ difference between them. The differences in average ranks between the proposed $R^2KM$ model and the existing models, including RVFL, RVFLwoDL, NF-RVFL, and RKM, are as follows: $0.67$, $3$, $1.67$, and $2.45$, respectively. The Nemenyi post-hoc test demonstrates that the proposed $R^2KM$ model is statistically superior to the baseline RVFLwoDL and RKM models. The $R^2KM$ model's lower ranking indicates enhanced generalization capabilities compared to the existing RVFL and NF-RVFL models. The high average RMSE and consistent performance across multiple statistical tests offer strong evidence that the proposed $R^2KM$ model surpasses the existing baseline models in terms of generalization.

\section{Conclusion}
\label{Conclusion}
In this paper, we propose a novel randomized based restricted kernel machine ($R^2KM$) to address the limitations of the RVFL network for hyperspectral image (HSI) classification. While RVFL improves stability by reducing the impact of random initialization in input-to-hidden layer weights, and enhances the model's capacity to capture complex non-linear relationships in data, $R^2KM$ further advances this by resolving the common issue of determining the optimal number of hidden nodes in RVFL networks. $R^2KM$ combines the computational efficiency of RVFL with the robust feature mapping capabilities of restricted kernel machines (RKM), offering a new approach for modeling intricate data interactions and non-linear patterns. Furthermore, the model utilizes a conjugate feature duality derived from the Fenchel-Young inequality to set an upper limit on the objective function, thereby improving its adaptability and scalability. This innovative duality, combined with $R^2KM$'s ability to represent kernel methods using both visible and hidden variables, significantly enhances the interpretability and robustness of kernel-based models. We assessed the performance of the proposed $R^2KM$ model using benchmark datasets from the UCI and KEEL repositories, comparing it against four state-of-the-art models for both classification and regression tasks. The results highlight the exceptional performance of the proposed $R^2KM$ model, which outperformed all baseline models, achieving an average accuracy improvement of up to $1.69\%$ over the second-best model. Furthermore, $R^2KM$ demonstrated outstanding effectiveness in regression tasks. Statistical analyses, including ranking, the Nemenyi post hoc test, the Friedman test, and the win-tie-loss sign test, confirm that our $R^2KM$ model significantly surpasses the baseline models in robustness and overall performance. Also, we carried out a series of experiments using four distinct hyperspectral image datasets to assess the effectiveness of the proposed model. The experimental results consistently showed that our model surpasses the performance of baseline approaches, achieving higher accuracy across all datasets. This improvement is especially evident in its ability to process the intricate spectral and spatial information inherent to hyperspectral images. By capturing both fine spectral variations and complex spatial structures, the model proves to be highly capable of addressing the unique challenges posed by hyperspectral image classification. In particular, hyperspectral images contain hundreds of narrow spectral bands, each providing detailed information about the materials in the scene. This makes classification difficult due to the high dimensionality and inter-band correlations. Traditional models often struggle with overfitting or missing important spectral and spatial features. Our proposed $R^2KM$ model, however, excels at extracting these features, balancing the trade-off between spectral resolution and spatial consistency. This approach enhances classification accuracy, particularly in scenarios where minor spectral variations are crucial for differentiating between classes. Future research could explore the development of adaptive methods that dynamically and efficiently adjust the parameters $\eta$ and $\lambda$ during the training process. This advancement would eliminate the need for manual tuning, enhancing the model's flexibility and ease of use. Additionally, applying the proposed model to other complex domains, such as time series forecasting or high-dimensional data, could open up new opportunities and showcase their versatility.

\section*{Acknowledgement}
This study receives support from the Science and Engineering Research Board (SERB) through the Mathematical Research Impact-Centric Support (MATRICS) scheme Grant No. MTR/2021/000787. The authors gratefully acknowledge the invaluable support provided by the Indian Institute of Technology Indore.
\bibliography{refs.bib}
\bibliographystyle{unsrtnat}
\end{document}